%% file: main.tex
\documentclass{article}

	\PassOptionsToPackage{round}{natbib}

\usepackage[final]{neurips_2020}

\input{packages}

\input{notations}

\newif\ifsup
\supfalse\suptrue		

\title{
	Online Algorithm for Unsupervised Sequential Selection with Contextual Information
}

\author{
	Arun Verma \\
	Department of IEOR\\
	IIT Bombay, India\\
	\texttt{v.arun@iitb.ac.in} \\
	\And
	Manjesh K. Hanawal \\
	Department of IEOR\\
	IIT Bombay, India\\
	\texttt{mhanawal@iitb.ac.in} \\
	\AND
	Csaba Szepesv\'ari \\
	DeepMind/University of Alberta\\
	Alberta, Canada\\
	\texttt{szepi@google.com} \\
	\And
	Venkatesh Saligrama  \\
	Departmetn of ECE\\
	Boston University, USA\\
	\texttt{srv@bu.edu} \\
}

\begin{document}
	\maketitle
	
	\begin{abstract}
        In this paper, we study {\em Contextual \underline{U}nsupervised \underline{S}equential \underline{S}election} (USS), a new variant of the stochastic contextual bandits problem where the loss of an arm cannot be inferred from the observed feedback. In our setup, arms are associated with fixed costs and are ordered, forming a cascade. In each round, a context is presented, and the learner selects the arms sequentially till some depth.  The total cost incurred by stopping at an arm is the sum of fixed costs of arms selected and the stochastic loss associated with the arm.  The learner's goal is to learn a decision rule that maps contexts to arms with the goal of minimizing the total expected loss. The problem is challenging as we are faced with an unsupervised setting as the total loss cannot be estimated. Clearly, learning is feasible only if the optimal arm can be inferred (explicitly or implicitly) from the problem structure. We observe that learning is still possible when the problem instance satisfies the so-called `Contextual Weak Dominance' $(\CWD)$ property. Under $\CWD$, we propose an algorithm for the contextual USS problem and demonstrate that it has sub-linear regret. Experiments on synthetic and real datasets validate our algorithm.
	\end{abstract}

	\section{Introduction}
	\label{sec:introduction}
	\input{introduction}

	\section{Problem Setting}
	\label{sec:problem_setting}
	\input{problem_setting}

		\subsection{Contextual Weak Dominance}			
		\label{ssec:learning}
		\input{learning}

	\section{Parameterization of Pairwise Disagreement  Probability}
	\label{sec:glmModel}
	\input{glmModel}

	\section{Algorithm for Contextual USS: \ref{alg:CUSS_WD}}
	\label{sec:algorithm}
	\input{algorithm}

	\section{Experiment}
	\label{sec:experiments}
	\input{experiment}

	\section{Conclusion and Future Directions}
    	We studied the unsupervised sequential selection problem with contextual information. It is a partial monitoring stochastic contextual bandit problem, where the loss of an arm can not be inferred from the observed feedback. But one can compare the feedback of two arms to see if they agree or disagree. We modeled the disagreement probability between each pair of the arms as linearly parameterized and developed an algorithm named \ref{alg:CUSS_WD} that achieves $O(\log T)$ regret with high probability. 
    	
    	We exploited the contextual information but ignored the inherent side observations due to the arms' cascade structure. By using the side observations, one can tighten the regret bounds. Another interesting future direction is to develop algorithms that decide whether it needs to go further down in the cascade when more information about context is revealed along the cascade.

	\section{Broader Impact}
	The work considered the unsupervised sequential selection problem with contextual information. While we are not targeting any specific applications, the work has many potential civilian applications. As usual, these can improve societal conditions, but of course, with any technology, specific deployments need care. However, this is outside of the scope of the present work, which is aimed at improving the basic algorithms and understand the fundamental challenges in this problem setting. Of course, the authors hope that their work will have an altogether positive impact, both by deepening our understanding of challenging sequential decision making under uncertainty and by potential future (careful) applications of the algorithms developed here. Having said this, we do not foresee any immediate negative impact of this work.

	\subsubsection*{Acknowledgments}
	Manjesh K. Hanawal would like to thank the support from INSPIRE faculty fellowships from DST, Government of India, SEED grant (16IRCCSG010) from IIT Bombay, and Early Career Research (ECR) Award from SERB. Csaba Szepesv\'ari gratefully acknowledges funding from the Canada CIFAR AI Chairs Program, Amii, and NSERC. Venkatesh Saligrama would like to acknowledge NSF Grants DMS -2007350 (VS), CCF-2022446, CCF-1955981, and the Data Science Faculty Fellowship from the Rafik B. Hariri Institute.

	\bibliographystyle{unsrtnat} 
	\bibliography{ref}

	\ifsup
		\onecolumn	
		\noindent\rule{\linewidth}{4pt} 
		\vspace{1.5mm}
			
		\centerline{\Large \bf Supplementary Material: `Online Algorithm for Unsupervised}
		\vspace{2mm}
		\centerline{\Large \bf Sequential Selection with Contextual Information'}
		\hrulefill \\

		\appendix
		\label{asec:Appendix}
		\input{appendix}

 		\subsection{Algorithm with Regularization Term}
 		\label{asec:lambda_algorithm}
 		\input{lambda_algorithm}

 		\section{Leftover details from \cref{sec:experiments}}
 		\label{sec:supp_experiments}
 		\input{supp_experiments} 	
 		
 		\vspace{4mm}
 		\hrule
	\fi

\end{document}

%% file: packages.tex
\usepackage[utf8]{inputenc} 
\usepackage[T1]{fontenc}    
\usepackage{hyperref}       
\usepackage{url}            
\usepackage{booktabs}       
\usepackage{amsfonts}       
\usepackage{nicefrac}       
\usepackage{microtype}      

\usepackage{graphicx}
\usepackage[dvipsnames]{xcolor}
\usepackage{rotating}
\usepackage{tabularx}
\usepackage{pdflscape}
\usepackage{amsmath,amsthm,amssymb}
\usepackage{algorithm,algpseudocode}
\usepackage{bbm,dsfont}
\usepackage{caption,subcaption}
\usepackage{cancel}
\usepackage{enumerate, cases}
\usepackage{thmtools,thm-restate}
\usepackage[none]{hyphenat}
\usepackage{natbib}
\usepackage{wrapfig}
\usepackage{mathtools}
\usepackage{array}
\usepackage{tabularx,multicol,multirow}
\newcolumntype{M}[1]{>{\centering\arraybackslash}m{#1}}
\usepackage{pifont}

\allowdisplaybreaks

\hypersetup{
	colorlinks	= true,		
	urlcolor     = blue,	 
	linkcolor	 = purple, 	 
	citecolor    = violet    	 
}
\usepackage[capitalize]{cleveref}

\allowdisplaybreaks

\algdef{SE}[DOWHILE]{Do}{doWhile}{\algorithmicdo}[1]{\algorithmicwhile\ #1}


%% file: notations.tex
\newcommand{\USS}{\cP_{\text{USS}}}
\newcommand{\PCWD}{\cP_{\text{CWD}}}

\newcommand{\WD}{\mathrm{WD}}
\newcommand{\CSD}{\mathrm{CSD}}
\newcommand{\CWD}{\mathrm{CWD}}

\newcommand{\TCSD}{\Theta_{\CSD}}
\newcommand{\TCWD}{\Theta_{\CWD}}

\newcommand{\bc}{\boldsymbol{c}}
\newcommand{\bq}{\boldsymbol{Q}}
\newcommand{\ist}{{i^\star_t}}

\newcommand{\Xt}{X_t}
\newcommand{\xt}{x_t}
\newcommand{\Yt}{Y_t}

\newcommand{\Yti}{Y_t^i}
\newcommand{\Ytj}{Y_t^j}
\newcommand{\Yts}{Y_t^\ist}

\newcommand{\F}{\cF}

\newcommand{\pijt}{p_{ij}^{(t)}}

\newcommand{\pist}{p_{i^\star_t I_t}^{(t)}}
\newcommand{\pijst}{p_{i^\star_t j}^{(t)}}

\newcommand{\pjist}{p_{ji^\star_t}^{(t)}}

\newcommand{\tpijt}{\tilde{p}_{ij}^{(t)}}
\newcommand{\tpjit}{\tilde{p}_{ji}^{(t)}}
\newcommand{\tplit}{\tilde{p}_{li}^{(t)}}

\newcommand{\tpiht}{\tilde{p}_{ih}^{(t)}}
\newcommand{\pihts}{p_{i^\star_t h}^{(t)}}

\newcommand{\Bt}{\mathcal{B}_t}

\newcommand{\Ts}{\theta_{ij}^\star}
\newcommand{\Te}{\hat\theta_{ij}^t}
\newcommand{\Tsl}{\theta_{li}^\star}
\newcommand{\Tel}{\hat\theta_{li}^t}

\newcommand{\Tsh}{\theta_{ih}^\star}
\newcommand{\Teh}{\hat\theta_{ih}^t}

\newcommand{\Vt}{V_{ij}^t}
\newcommand{\oVt}{\overline{V}_{ij}^t}
\newcommand{\oVijt}{\overline{V}_{ij}^t}

\newcommand{\Vinv}{(V_{ij}^t)^{-1}}
\newcommand{\Vtinv}{(V_{ij}^t)^{-1}}
\newcommand{\oVtinv}{(\overline{V}_{ij}^t)^{-1}}
\newcommand{\Vtl}{V_{li}^t}

\newcommand{\Vinvl}{(V_{li}^t)^{-1}}

\newcommand{\Vth}{V_{ih}^t}
\newcommand{\Vths}{V_{i^\star_t h}^t}

\newcommand{\Vinvh}{(V_{ih}^t)^{-1}}

\newcommand\numberthis{\addtocounter{equation}{1}\tag{\theequation}}

\newcommand{\Regret}{\mathfrak{R}}

\newcommand{\EE}[1]{\bE\left[#1\right]}

\newcommand{\Prob}[1]{\bP\left\{#1\right\}}

\newcommand{\R}{\bR}

\newcommand{\one}[1]{\mathds{1}_{\left\{#1\right\}}}

\renewcommand{\phi}{\varphi}
\renewcommand{\epsilon}{\varepsilon}

\newcommand{\norm}[1]{\left\|#1\right\|}

\newcommand{\al}[1]{ \begin{align} #1  \end{align}}
\newcommand{\eq}[1]{ \begin{equation} #1  \end{equation}}
\newcommand{\als}[1]{ \begin{align*} #1  \end{align*}}
\newcommand{\eqs}[1]{ \begin{equation*} #1  \end{equation*}}

\newcommand{\el}{\end{flushleft}}
\newcommand{\bl}{\begin{flushleft}}

\newcommand{\argmin}{\arg\!\min}

\newcommand{\bE}{\mathbb{E}}

\newcommand{\bP}{\mathbb{P}}

\newcommand{\bR}{\mathbb{R}}

\newcommand{\cF}{\mathcal{F}}

\newcommand{\cP}{\mathcal{P}}

\newcommand{\cX}{\mathcal{X}}

\theoremstyle{plain}

\newtheorem{lem}{Lemma}

\newtheorem{rem}{Remark}
\newtheorem{defi}{Definition}
\newtheorem{assu}{Assumption}

%% file: introduction.tex

Industrial systems, such as those found in medical, airport security, and manufacturing, utilize a suite of tests or classifiers for monitoring patients, people, and products. Tests have costs with the more intrusive and informative ones resulting in higher monetary costs and higher latency. For this reason, they are often organized as a classifier cascade \citep{chen:2012, AISTATS13_trapeznikov2013supervised, NIPS15_wang2015efficient}, so that new input is first probed by an inexpensive test then a more expensive one. The goal of a cascaded system is to resolve easy to handle examples early so that the overall system maintains high accuracy at low average costs.

Over time, due to environmental changes or test calibrations, sequential testing protocols (STP) may no longer be accurate, resulting in higher costs. While one can leverage off-line methods such as supervised training of cascades ~\citep{NIPS15_wang2015efficient}, they require new annotated data collection. In many scenarios, new data cannot be collected in-situ, and system shutdown is not an option. In the absence of annotated data, we face a dilemma. While we can observe test outcomes, we cannot ascertain their reliability due to the absence of ground truth, necessitating {\it unsupervised sequential selection (USS)} methods, where an arm represents a test/classifier. Recent works~\citep{AISTATS17_hanawal2017unsupervised,AISTATS19_verma2019online, ACML20_verma2020thompson} propose methods for solving the USS problem; however, they focus exclusively on the non-contextual setting, which in essence requires inputs (people, objects, or products) to be homogeneous, and as such, these methods are unrealistic since contexts (high vs. low risk) can guide the arm selection.

In this context, we propose the contextual USS. In our setup, inputs arrive sequentially, and the learner observes a continuous-valued context as input. While the learner knows the costs of each arm, he does not know the associated stochastic loss. Furthermore, the learner does not benefit from feedback from his arm selection, in contrast to the conventional contextual bandit works \citep{AISTATS11_ContextualBandit_BeygelzimerSchapire}. Thus, while being agnostic to the true loss, the learner must sequentially choose the arm that leads to the smallest total loss, where the total loss is the sum of the cost of using an arm and the mean loss associated with the arm. As such, our proposed problem is a special case of the stochastic partial monitoring problem with contextual inputs \citep[Chapter 37]{LaSze19:book}. 
Most of the prior work on partial monitoring problem is restricted to cases where observed feedback can identify the losses for selected actions. However, in many areas like crowd-sourcing \citep{NIPS17_bonald2017minimax, ICLM18_kleindessner2018crowdsourcing}, resource allocation \citep{NeurIPS19_verma2019censored},  medical diagnosis \citep{COMSNETS20_verma2020unsupervised}, and many others, feedback from actions may not even be sufficient to identify the losses.

While we draw upon several concepts introduced in earlier work \citep{AISTATS17_hanawal2017unsupervised}, there are additional challenges in the contextual case due to the unsupervised nature of the problem. First, unlike vanilla-USS, the loss here is context-dependent. We propose notions of contextual weak dominance as a means to relate observed disagreements to differences in losses between any two arms. We then propose a parameterized Generalized Linear Model (GLM) to model the context-conditional disagreement probability between any two arms and validate the model empirically.

A fundamental technical challenge is in the estimation of disagreement probabilities uniformly across all contexts in the finite time while ensuring sufficient exploration between different arm selection protocols, required for honing in on the optimal selection strategy. In particular, since contexts are continuous-valued, and because we have no control over inputs, the contextual observations, in the finite time, may not persistently span the whole space, and estimates are often unreliable. To this end, we adapt techniques from parameterized contextual bandits \citep{AISTATS11_chu2011contextual,ICML17_li2017provably} for our unsupervised setting. We propose an algorithm based on the principle of optimism, namely, the larger indexed arm in cascade is chosen when uncertain. We show that our algorithm navigates the exploration-exploitation tradeoffs in different ways and lead to sub-linear cumulative regret. We then validate it on several problem instances derived from synthetic and real datasets.

\paragraph{Related Work.}
{\em Stochastic Contextual multi-armed Bandits (SCB):} In each round, the learner observes the context and decides which arm, among a finite number of arms, to apply \citep{AISTATS11_ContextualBandit_BeygelzimerSchapire}. By playing an arm, the learner observes a stochastic reward that depends on the context and the arm selected. The most commonly studied model assumes that each arm is parameterized, and the mean reward of an arm is the inner product of the context and an unknown parameter associated with the arm. Contextual bandits have been applied to problems ranging from online advertising \citep{WWW10_li2010contextual,AISTATS11_chu2011contextual} and recommendations \citep{NIPS08_langford2008epoch} to clinical trials \citep{JASA79_woodroofe1979one} and mobile health \citep{MB17_tewari2017ads}.
{\em Generalized linear models (GLM)} assume that the mean reward is a non-linear link function of the inner product between the context vector and the unknown parameter vector \citep{NIPS10_filippi2010parametric,ICML17_li2017provably}. GLMs are also useful models for the classification problems where rewards, in the context of online learning problems, could be binary \citep{ICML16_zhang2016online,NIPS17_jun2017scalable}. A more challenging non-parameterized version of the stochastic contextual bandits is studied in \citep{ICML14_TamingTheMonster_AgarwalSchapire}. 

Another framework that is closely related to SCB is {\em stochastic linear bandits (SLB)} \citep{JMLR02_auer2002using,COLT08_dani2008stochastic,MOR10_rusmevichientong2010linearly,NIPS11_abbasi2011improved}. In this setup, the environment is parameterized, and there could be uncountably many arms (within some bounded radius), also referred to as decision set. The arms are characterized into their feature vectors, and the mean reward for playing an arm is given as the inner product of the parameter (unknown) and the feature vector associated with the arm. In situations where the decision set is allowed to vary in each round and are finite, SLBs are equivalent to SCBs, where feature vectors correspond to context-arm pairs \citep{WWW10_li2010contextual,ICML17_li2017provably}. 
For our work, we leverage GLMs as models for disagreement probability between any two arms. While it is tempting to reduce contextual USS to SCBs, note that, unlike prior works, we do not observe loss for our action choices, and so conventional algorithms such as LinUCB and UCB-GLM \citep{WWW10_li2010contextual,ICML14_TamingTheMonster_AgarwalSchapire,ICML17_li2017provably} cannot be applied.

Most of the prior work \citep{AISTATS17_hanawal2017unsupervised, AISTATS19_verma2019online, ACML20_verma2020thompson} considered the problem of learning an optimal action but ignored the contextual information. In this work, we incorporated contextual information, which is readily available in many applications.  Exploiting the {\it real-valued} contextual information (features) for improving the arm selection strategy is non-trivial due to the unsupervised nature of the problem where the standard analysis of contextual bandits does not apply. We made necessary modeling assumptions to leverage GLMs to parameterize the disagreement probability between two arms and extended the existing definitions to address the new setup's learnability issues. However, the problem still requires new ideas and analysis methods to derive an efficient algorithm, which poses new technical challenges for analysis.

%% file: problem_setting.tex

We consider a stochastic contextual bandits problem with $K$ arms. The set of arms is denoted as $[K]$ where $[K] \doteq \{1,2,\ldots, K\}$. In each round $t$, the environment generates a vector $\left(\Xt, \Yt, \{\Yti\}_{i \in [K]}\right)$. The vector $X_t$ denotes the context in round $t$ and forms an independent and identically distributed (IID) sequence drawn from a bounded set $\cX \subset \R^d$ according to an unknown but fixed distribution $\nu$. The binary reward for context $X_t$ is denoted by $Y_t \in \{0,1\}$, which is hidden from the learner. The vector $\left(\{\Yti\}_{i \in [K]}\right) \in \{0, 1\}^{K}$ represents observed feedback at time $t$, where $\Yti$ denotes the feedback observed after playing arm $i$ with $X_t$ as input\footnote{In our setup, an arm $i$ could be a classifier that outputs label $Y^i$. The classifier's input could be a context and any combinations of feedback observed from classifiers coming before the arm $i$ in the cascade. For example, consider a case where each arm represents a crowd-sourced worker. After using the first $i$ crowd-sourced workers, the final label can be a function of predicted labels of the first $i$ crowd-sourced workers.}. We denote the cost for using arm $i$ as $c_i\geq 0$ that is known and the same for all contexts.

In contextual USS, the arms are assumed to be ordered and form a cascade. When the learner selects an arm $i \in [K]$, the feedback from all arms till arm $i$ in the cascade are observed. The expected loss of playing the arm $i$ for a given context $\xt$ is denoted as $\gamma_i(\xt) \doteq \EE{\one{\Yti \neq \Yt|X=\xt}} = \Prob{\Yti \neq \Yt|X=\xt}$, where $\one{A}$ denotes indicator of event $A$. For soundness, we assume that the probability density function of context distribution is strictly positive on $\mathcal{X}$ such that the conditional probabilities are well defined. The total expected loss incurred by playing arm $i$ for context $\xt$ is defined as $\gamma_i(\xt)+\lambda_iC_i$, where $C_i \doteq c_1 + \ldots + c_i$ and $\lambda_i$ is a trade-off parameter that normalizes the incurred cost and the loss of playing arm $i$.

Since the true rewards are hidden from the learner, the expected loss of an arm cannot be inferred from the observed feedback. We thus have a version of the stochastic partial monitoring problem (\cite{MOR06_cesa2006regret,ALT12_bartok2012partial,MOR14_bartok2014partial}, and we refer to it as contextual \underline{u}nsupervised \underline{s}equential \underline{s}election (USS). Let $\bq$ be the unknown joint distribution of $(X, Y, Y^1,Y^2 \ldots, Y^K)$. Henceforth we identify a contextual USS instance as $P \doteq (\bq,\bc)$ where $\bc \doteq (c_1, c_2, \ldots, c_K)$ is the known cost vector of arms. We denote the collection of contextual USS instances as $\USS$. For instance $P \in \USS$, the optimal arm for a context $x_t$  is given as follows: 
\eq{
	\label{equ:optimalArm}
	\ist \in \max\left\{\argmin _{i \in [K]} \left( \gamma_i(\xt)+\lambda_iC_i \right)\right\},
}
where the choice of $\ist$ is risk-averse as we prefer the arm with lower error among the optimal arms. 

The interaction between the environment and a learner is given in Algorithm \ref{alg:ContxUSS}.
\begin{algorithm}[H]
	\caption{Learning on contextual USS instance $(\bq, \bc)$}
	\label{alg:ContxUSS}
	For each round $t$: 
	\begin{enumerate}
		\item \textbf{Environment} chooses a vector $(X_t, \Yt, \{\Yti\}_{i \in [K]})\sim \bq$.
		\item \textbf{Learner} observes a context $X_t=\xt$ and selects an arm $I_t \in [K]$ to stop in cascade.
		\item \textbf{Feedback and Loss:} The learner observes feedback $(\Yt^1, \Yt^2, \ldots, \Yt^{I_t})$ and incurs a total loss $\one{\Yti \neq \Yt|X=\xt} + \lambda_{I_t}C_{I_t}$. 
	\end{enumerate}
\end{algorithm}

The learner's goal is to find an arm for each context such that the cumulative expected loss is minimized. Specifically, for $T$ contexts, we measure the performance of a policy that selects an arm $I_t$ for a context $x_t$ in terms of regret given by
\eq{
	\label{equ:cum_regret}
	\Regret_T = \sum_{t=1}^T\left( \gamma_{I_t}(\xt) + \lambda_{I_t}C_{I_t} - \left(\gamma_\ist(\xt) + \lambda_\ist C_{i^\star_t} \right) \right).
}

We seek policies that yield sub-linear regret, i.e., $\Regret_T/T \rightarrow 0$ as $T \rightarrow \infty$. It implies that the learner collects almost as much reward in the long run as an oracle collects that knew the optimal arm for every context. We say that a problem instance $P \in \USS$ is learnable if there exists a policy such that $\lim\limits_{T \rightarrow \infty}\Regret_T/T = 0$.

In the sequel, we discuss the selection criteria for optimal arm for a given context and the conditions under which instances of $\USS$ are learnable.

%% file: learning.tex

Next, we introduce the contextual weak dominance property of a problem instance. 
\begin{defi}[Contextual Weak Dominance $(\CWD)$] 
	\label{def:CWD} 
	Let $\ist$ denote optimal arm for context $\xt$. Then the context $\xt$ is said to satisfy weak dominance $(\WD)$ property if
	\eq{
		\label{equ:WDProp}
		\forall j>\ist: C_j - C_\ist > \Prob{\Yts \ne \Ytj |X = \xt}.	
	}
	A problem instance $P \in \USS$  is said to satisfy the $\CWD$ property if all contexts of $P$ satisfy $\WD$ property. We denote the set of all instances in $\USS$ that satisfies $\CWD$ property by $\PCWD$.
\end{defi}

In the following, we use an alternative characterization of the $\CWD$ property, given as
\eq{
    \label{def:Xi}
    \xi(x_t) \doteq \min_{j>\ist}\left\{C_j - C_\ist - \Prob{\Yts \ne \Ytj|X = \xt} \right\} > 0.
}
We define $\xi \doteq \inf_{x \in \cX} \xi(x)$ and assume that $\xi>0$. The larger the value of $\xi$, `stronger' is the $\CWD$ property, and easier it is to identify an optimal arm for given contexts. We later characterize the regret upper bounds of proposed algorithms in terms of $\xi$. We also discuss the case when a fraction of contexts satisfies $\WD$ property in the supplementary material.

\subsection{Selection Criteria for Optimal Arm}
Without loss of generality, we set $\lambda_i=1$ for all $i\in [K]$ as their value can be absorbed into the costs. Since $\ist = \max\big\{\arg\min\limits_{i \in [K]}\left(\gamma_i(x_t)+ C_i \right)\big\}$, it must satisfy following equation:
\begin{subequations}
	\label{eq:cost_exp_err}
	\al{
		&\forall j<\ist \,:\, C_\ist - C_j \leq \gamma_j(\xt)-\gamma_\ist(\xt) \,, \label{eq:cwd1}\\ 
		&\forall j>\ist \,:\, C_j - C_\ist > \gamma_\ist(\xt) - \gamma_j(\xt) \,. \label{eq:cwd2}
	}
\end{subequations}

As the loss of an arm is not observed, the above equations can not lead to a sound arm selection criteria. We thus have to relate the unobservable quantities in terms of the quantities that can be observed. In our setup, we can compare the feedback of two arms, which can be used to estimate the disagreement probabilities between them. For notation convenience, we define $\pijt \doteq \Prob{\Yti \ne \Ytj|X=\xt}$ for $i<j$. The value of $\pijt$ can be estimated as it is observable.
Our next result bounds unobserved error rates differences in terms of their observable disagreement probabilities for a given context.
\begin{restatable}{lem}{ErrProbContx}
	\label{lem:err_prob_contx}
	For any  $i$, $j$, and $x_t\in \mathcal{X}$, $\gamma_{i}(\xt) - \gamma_{j}(\xt) = \pijt - 2\Prob{\Yti = \Yt, \Ytj \ne \Yt| X = \xt}$. 
\end{restatable}
The detailed proof of \cref{lem:err_prob_contx} and all other missing proofs appear in the supplementary material.

Now, using \cref{lem:err_prob_contx}, we can replace  \cref{eq:cwd1} by
\eq{
	\label{eq:selectDisProbLow}
	\forall j<\ist \,:\, C_{\ist} - C_j \leq  \pjist,
}
which only has observable quantities. 
For $j>\ist$, using the $\CWD$ property, we replace \cref{eq:cwd2} by
\eq{
	\label{eq:selectDisProbHigh}
	\forall j>\ist \,:\, C_j - C_{\ist} >  \pijst.
}

Using \cref{eq:selectDisProbLow} and \cref{eq:selectDisProbHigh}, our next result gives the optimal arm for a given context $x_t$.
\begin{restatable}{lem}{SetBx}
	\label{lem:Bx}
	Let $P \in \PCWD$ and $\Bt = \left\{i: \forall j>i, C_j - C_i > \pijt \right\}\cup \{K\}$. Then the arm $I_t =\min(\Bt)$ is the optimal arm for a context $x_t$.
\end{restatable}
By construction, the optimal arm lies in set $\Bt$. Because of \cref{eq:selectDisProbLow}, any sub-optimal arm having smaller index than optimal arm do not satisfy \cref{eq:selectDisProbHigh}, hence it can not be in set $\Bt$. Therefore, the smallest arm of set $\Bt$ is the optimal arm. 
\begin{restatable}{thm}{learnCWD}
	\label{thm:learnCWD}
	The set $\PCWD$ is maximal learnable. 
\end{restatable}
The proof establishes that under the $\CWD$ property, there exists a `sound' arm selection policy that identifies the optimal arm for each context. The sound policy only uses conditional disagreement probabilities between pairs of arms that can be estimated from the feedback of arms.

%% file: glmModel.tex

Since the number of contexts could be much larger (can be infinite) than the learning horizon, in stochastic contextual bandits, a correlation structure is assumed between the reward (loss) and the contexts \citep{JMLR02_auer2002using, WWW10_li2010contextual, ICML17_li2017provably}. It is often realized via parameterization of the arms such that expected rewards (or losses) observed from an arm depend on the unknown parameter. In our setting, we cannot observe a loss for any arm. Hence parameterization of an expected loss of the arms is not useful.  However, we can obtain feedback of two arms for a given context and can compare them. For example, we can check whether two arms' feedback agrees or disagrees for a given context. Thus, we assume a correlation structure on the disagreement probability for a pair of arms across the contexts and parameterize it using generalized linear models. For $i<j$ and context $x_t$, the disagreement probability for $(i,j)$ pair of arms is given via a function $\mu$ as follows:
\eq{
	\label{equ:disProb}
	\Prob{\Yti \ne \Ytj|X=\xt}=\mu(\Phi_{ij}(\xt)^\top \Ts), 
} 
where $\xt \in \R^{d}$, $\Phi_{ij}: \R^{d} \rightarrow \R^{d^\prime}$ is a feature map for some $d^\prime \ge d$,\footnote{Let $\R^{d_{ij}}$ be the space where \cref{equ:disProb} holds for $(i,j)$ pair of arms, and $\Phi_{ij}$ is the feature map that lift $\xt$ from $\R^{d}$ space to $\R^{d_{ij}}$ space. For simplicity, we take $d^\prime = \max_{\forall i < j \le K} d_{ij}$.} and $\Ts \in\R^{d^\prime}$ is the unknown parameter for $(i,j)$ pair.

We assume the following assumptions on context distribution $\nu$ and function $\mu$, which is standard in the GLM bandit literature \citep{NIPS10_filippi2010parametric,ICML17_li2017provably}:
\begin{assu}[GLM]
	\begin{itemize}
		\setlength{\itemsep}{0pt}
		\setlength{\parskip}{3pt}
		\setlength{\itemindent}{-1.75em}
		\item For all $x \in \mathcal{X}$ and $(i,j)$ pairs, $\norm{\Phi_{ij}(x)}_{2} \le 1$.
		\item $\kappa \doteq \inf_{\norm{ x }_2 \le 1, \norm{ \theta - \Ts }_2 \le 1}$ $\dot{\mu}(\Phi_{ij}(x)^\top \theta) > 0$ for all $(i,j)$ pairs.
		\item There exists a constant $\lambda_\Sigma > 0$ such that $\lambda_{min}\left( \EE{\Phi_{ij}(X)\Phi_{ij}(X)^\top} \right) \ge \lambda_\Sigma$ for all $(i,j)$ pairs.
		\item The function $\mu : \R \rightarrow [0,1]$ is continuously differentiable and Lipschitz with constant $k_\mu$. 
	\end{itemize}	
\end{assu}

For our setting, the function $\mu$ is defined as $\mu(z) ={1}/{(1+\mathrm{e}^{-z})}$, which is the logistic function. The logistic function is widely used function for binary classification model and has $k_\mu \le 1/4$.

In contextual USS setup, we can compare the arms' feedback and check whether they agree or not for a given context. These binary observations (agree or disagree) can be treated as noisy samples of the disagreement probability. The noise in the binary observation obtained by comparing the feedback of $(i, j)$ pair of arms in round $t$, is given by 
\als{
	\epsilon_{ij}^{(t)}=
	\begin{cases}
		1 - \mu(\Phi_{ij}(\xt)^\top \Ts ), & \text{with probability } \mu(\Phi_{ij}(\xt)^\top \Ts ) \\
		 - \mu(\Phi_{ij}(\xt)^\top \Ts ), & \text{with probability } \left(1-\mu(\Phi_{ij}(\xt)^\top \Ts) \right)
	\end{cases}	
}
where $\epsilon_{ij}^{(t)}$ is $\F_{t}$-measurable with $\EE{\epsilon_{ij}^{(t)}|\F_t}=0$. Here $\F_t$ denotes sigma algebra generated by history $\left\{\left(X_s, I_s, \left\{Y_s^i\right\}_{i \in [I_s]}\right) \right\}_{s \in [t]}$ till time $t$. Since $\epsilon_{ij}^{(t)}$ is a zero-mean shifted Bernoulli random variable, $\epsilon_{ij}^{(t)}$ satisfies the following sub-Gaussian condition with parameter $\sigma \in (0,1)$:
\eqs{
	\EE{\exp(\lambda\epsilon_{ij}^{(t)})|\F_{t}} \le \exp\left(\frac{\lambda^2\sigma^2}{2}\right), \hspace{2mm} \forall \lambda \in \R.
}

Let $d_{ij}(t) \doteq \one{\Yti \neq \Ytj|X=\xt}$ be the disagreement indicator for a context $\xt$ and $S_{ij}^t$ be the set of indices of contexts for which disagreements are observed for $(i,j)$ pair of arms till round $t$. In round $t$, we estimate $\Ts$, denoted by $\Te$, using the following equation adapted from the maximum likelihood estimator (MLE) used for GLM bandits \citep{NIPS10_filippi2010parametric,ICML17_li2017provably}:
\al{
	\sum_{s\in S_{ij}^t}&\left(d_{ij}(s) - \mu(\Phi_{ij}(x_s)^\top\theta)\right) \Phi_{ij}(x_s) = 0. \label{equ:glm_est}
}

In the next section, we develop an algorithm that exploits \cref{lem:Bx} for selecting the optimal arm to each context. The algorithm replaces the terms $\pijt$ in \cref{lem:Bx} by their optimistic estimates.

%% file: algorithm.tex

Our algorithm, named \ref{alg:CUSS_WD}, is based on the {\it optimism-in-the-face-of-uncertainty} (OFU) principle. \ref{alg:CUSS_WD} works as follows: It takes $\delta$ and $m$ as inputs, where $\delta$ is the confidence in the estimated parameters and used for computing confidence bound for $\Ts$ as given by \cref{lem:theta_est_glm}. The choice of $m$ ensures that with probability at least $(1-\delta)$, the sample correlation matrix $V_{ij}^t=\sum_{s \in S_{ij}^t}\Phi_{ij}(x_s)\Phi_{ij}(x_s)^\top$ for each $(i,j)$ pair where $i<j$, is invertible. A high probability upper bound on $m$ is computed using \cref{lem:min_eigen_lb}. The algorithm collects feedback from all arms by selecting the arm $K$ irrespective of the context received for first $m$ rounds. After $m$ rounds, the sample correlation matrix and the estimate of $\Ts$ are computed for each $(i,j)$ pair where $i<j$.

For $t>m$, the learner receives a context $x_t$ and plays the arm $i=1$ and then observe its feedback. For each $(i,j)$ pair and context $\xt$, the upper bound on disagreement probability $\tpijt$ is computed using $\Te$ and confidence bonus $\alpha_{ij}^t\norm{ \Phi_{ij}(x_t)}_{\Vinv}$. Here the notation $\norm{x}^2_A \doteq x^\top Ax$ denotes the weighted $l_2$-norm of vector $x \in \R^d$ with respect to a positive definite matrix $A \in \R^{d\times d}$. The confidence bonus has two terms. The first term $\alpha_{ij}^t$ is a slowly increasing function in $t$ whose value is specified in Lemma \ref{lem:theta_est_glm}, and the second term $\norm{ \Phi_{ij}(x_t)}_{\Vinv}$ decreases to zero as $t$ increases. 

\begin{algorithm}[H]
    \renewcommand{\thealgorithm}{USS-PD}
    \floatname{algorithm}{}
    \caption{Algorithm for Contextual USS using Pairwise Disagreement}
	\label{alg:CUSS_WD}
	\begin{algorithmic}[1]
		\State \textbf{Input:} Tuning parameters: $\delta \in (0,1)$ and $m>0$
		\State Select arm $K$ for first $m$ contexts
		\State $\forall i< j \le K:$ set ${V}_{ij}^{m} \leftarrow \sum_{t=1}^{m}\Phi_{ij}(x_{t}) {\Phi_{ij}(x_{t})}^\top$ and update $\hat\theta_{ij}^m$ by solving \cref{equ:glm_est} 
		\For{$t= m+1, m+2, \ldots$}
			\State Receive context $x_t$. Set $i=1$ and $I_t=0$
			\Do
				\State Play arm $i$ 
				\State $\forall j \in [i+1, K]:$ compute $\tpijt \leftarrow \mu\left(\Phi_{ij}(\xt)^\top \hat\theta_{ij}^{t-1} + \alpha_{ij}^{t-1}\norm{ \Phi_{ij}(\xt)}_{\left({V}_{ij}^{t-1}\right)^{-1}} \right)$
				\State If $\forall j \in [i+1, K]: C_j - C_i > \tpijt$ or $i=K$ then set $I_t = i$ else  set $i=i+1$
			\doWhile{$I_t=0$}
			\State Select arm $I_t$ and observe $Y_t^1, Y_t^2, \dots, Y_t^{I_t}$
			\State $\forall i< j \le I_t:$  update $\Vt \leftarrow {V}_{ij}^{t-1} + \Phi_{ij}(x_{t}) {\Phi_{ij}(x_{t})}^\top$ and $\Te$ by solving \cref{equ:glm_est} 
		\EndFor
	\end{algorithmic}
\end{algorithm}

After computing $\tpijt$, the algorithm checks whether the arm $i$ is the best arm using \cref{eq:selectDisProbHigh} with $\tpijt$ in place of $\pijt$. If the arm $i$ is not the best, then the algorithm plays the next arm, and then the same process is repeated. If the arm $i$ is the best arm for context, then the algorithm stops at that arm with $I_t =i$ for that context. After selecting arm $I_t$, the feedback from arms $1, \ldots, I_t$ are observed. After that, the values of $V^t_{ij}$ are updated, and $\Te$ are re-estimated. The same process is repeated for subsequent contexts.

\begin{rem}
    GLM bandits are well studied but require reward or loss information. In the USS setup, loss of selected arm can not be observed; hence finding the optimal arm is challenging. Due to binary disagreement, \ref{alg:CUSS_WD} uses the MLE estimator for $\Ts$ as used in GLM bandits \citep{NIPS10_filippi2010parametric,ICML17_li2017provably}. However, the feedback structure and the way arms are selected in the USS setup differ from that in the GLM bandits. Further, our analysis needs carefully connecting the regret with the bad events that make \ref{alg:CUSS_WD} selects non-optimal arms.
\end{rem}

\begin{rem}
    We force the algorithm to explore until the correlation matrix $V_{ij}^t$ is invertible for all $(i,j)$ pairs. The invertibility can also be ensured by adding a regularization term \citep{NIPS11_abbasi2011improved,ICML16_zhang2016online,NIPS17_jun2017scalable} to avoid forced exploration. However, the analysis of \ref{alg:CUSS_WD} with regularization term still required to the non-regularized part of the sample correlation matrix becomes invertible. See the supplementary material for the algorithm and its analysis. 
\end{rem}

\subsection{Regret Analysis of \ref{alg:CUSS_WD}}

The following definition is useful in our regret analysis. 
\begin{defi}[Arm Preference ($\succ_t$)]
	\label{def:preference} \ref{alg:CUSS_WD} prefers an arm $i$ over $j$ for context $x_t$ if 
	\begin{subnumcases}	
		{i \succ_t j \doteq }
		C_i - C_j < \tpjit,  &\textnormal{if $j<i$} \label{def_prefer_l} \\
		C_j - C_i > \tpijt,	&\textnormal{if $j>i$} \label{def_prefer_h}.  
	\end{subnumcases}
\end{defi}

Our next result bounds the number of disagreement observations required from a pair of arms say $(i,j)$, such that the smallest eigenvalue of its sample correlation matrix $V_{ij}$ matrices is larger than a fixed value. This result uses the standard results from random matrix theory \citep{Book_vershynin_2012}.
\begin{restatable}{lem}{minEigenLB}
	\label{lem:min_eigen_lb}
	Let ${V}_{ij}^{t} = \sum_{s \in S_{ij}^t}\Phi_{ij}(x_{s}) {\Phi_{ij}(x_{s})}^\top$, $\Sigma_{ij} = \EE{\Phi_{ij}(X)\Phi_{ij}(X)^\top}$, $\Psi$ and $\delta \in (0,1)$ be two positive constants. Then, there exist positive universal constants $C_1$ and  $C_2$ such that the minimum eigenvalue of $\lambda_{min}({V}_{ij}^{t}) \ge \Psi$ with probability at least $1-2\delta/K^2$, iff	
	\eqs{
		|S_{ij}^t| \ge \left(\frac{C_1\sqrt{d^\prime}+C_2\sqrt{\log (K^2/2\delta)}}{\lambda_{min}(\Sigma_{ij})}\right)^2+ \frac{2\Psi}{\lambda_{min}(\Sigma_{ij})}.
	}
\end{restatable}

The next result is adapted to our setting from the confidence bounds for  maximum likelihood estimator used in GLM bandits \citep{ICML17_li2017provably}.
\begin{restatable}[Confidence Ellipsoid]{lem}{thetaEstGLM}
	\label{lem:theta_est_glm} 
	Let $m$ be such that $\lambda_{min}(V_{ij}^{m+1}) \ge 1$ for any pair $(i,j)$. Then the following event holds with probability at least $1-2\delta/K^2$ for \ref{alg:CUSS_WD}:
	\begin{align*}
	&\norm{ \Te - \Ts }_{\Vt} \le \alpha_{ij}^t, \;\forall t > m 
	\end{align*}
	where $\alpha_{ij}^t = \frac{2\sigma}{\kappa}\sqrt{\frac{d^\prime}{2}\log\left(1 + \frac{2t}{d^\prime} \right) + \log\left(\frac{K^2}{2\delta}\right)}$.
\end{restatable}

The regret analysis of GLM bandits hinges on bounding the instantaneous regret in each round, which is tied to the estimation error of the GLM parameters. Due to the unsupervised setting and cascade structure, this way of bounding regret does not work in our setup. Our analysis goes by bounding the number of pulls of the sub-optimal arms. However, unlike standard bandits, we have to distinguish whether the sub-optimal arm pulled by \ref{alg:CUSS_WD} is on the `left' or `right' of the optimal arm in the cascade. It requires our analysis to handle both the cases carefully. Since \ref{alg:CUSS_WD} uses a similar MLE estimator for parameter estimation as in GLM bandits, we only adapt their asymptotic normality results. Our next results give conditions when \ref{alg:CUSS_WD} prefers a sub-optimal arm for a context. 
\begin{restatable}{lem}{subOptimalLowerSelection}
	\label{lem:subOptimalLowerSelection}
	Let $\theta \in \TCWD$. Then \ref{alg:CUSS_WD} prefers any sub-optimal arm $l < i^\star_t$ for context $x_t$ with probability at most $\delta/2$.
\end{restatable}

\begin{restatable}{lem}{subOptimalUpperSelection}
	\label{lem:subOptimalUpperSelection}
	Let $\theta \in \TCWD$. If \ref{alg:CUSS_WD} prefers a sub-optimal arm $h > i^\star_t$ for context $x_t$ then
	\begin{equation*}
		2k_\mu \alpha_{i^\star_t h}^t > \xi_{i^\star_t h}(x_t)\sqrt{\lambda_{min}(\Vths)}.
	\end{equation*} 
	where $\xi_{\ist h} = C_h - C_\ist - \pihts$ and $\alpha_{ij}^t$ is given by \cref{lem:theta_est_glm}.
\end{restatable}

Let $m \doteq C\lambda_\Sigma^{-2}\left(d^\prime+\log(k^2/2\delta)\right) + 2\lambda_\Sigma^{-1}$, where $C>0$ is the universal constant and $R_{max} \doteq \max_{i\in [K], x\in \mathcal{X}}$ $ \left[C_i +  \gamma_{i}(x) - \left( C_{i^\star} + \gamma_{i^\star}(x)\right)\right]$, where $i^\star$ is the optimal arm for context $x$.
Now we state the regret upper bound of \ref{alg:CUSS_WD}.
\begin{restatable}[Regret Upper Bound]{thm}{cumRegGLM}
	\label{thm:cum_reg_glm}
	Let $\theta \in \TCWD$, $\delta \in (0,1)$, Assumption 1 holds, and $\xi_h = \min\limits_{t \ge 1} \xi_{i^\star_t h}(x_t)$. Then with probability at least $1-2\delta$, the regret of \ref{alg:CUSS_WD} for $T > m$ contexts is
	\begin{align*}
		\Regret_T &\le R_{max}\Bigg[m + \sum_{h=2}^{K} \Bigg(\hspace{-1mm} \Bigg(\frac{C_1\sqrt{d^\prime} + C_2\sqrt{\log \left(\frac{K^2}{2\delta}\right)}}{\lambda_{\Sigma}}\Bigg)^2  \hspace{-2mm} + \frac{16}{\lambda_{\Sigma}}\\
		&\qquad\left(\frac{k_\mu\sigma}{\xi_{h}\kappa}\right)^2 \left(\frac{d^\prime}{2}\log\left(1 + \frac{2T}{d^\prime} \right) + \log\left(\frac{K^2}{2\delta}\right) \right)\Bigg)\Bigg].
	\end{align*} 
\end{restatable}

\begin{restatable}{cor}{orderCumRegGLM}
	\label{cor:cum_reg_glm}
	Let technical conditions stated in \cref{thm:cum_reg_glm} hold. Then with probability at least $1-2\delta$
	\begin{equation*}
        \Regret_T \le O\left({Kd^\prime\log(T)}/{\xi^2}\right).	    
	\end{equation*}
\end{restatable}

The regret of \ref{alg:CUSS_WD} for instance $\theta \in \TCWD$ is logarithmic in $T$ and grows linearly with $d^\prime$ and $K$. The regret is inversely dependent on the value of $\xi \doteq \min\limits_{h\ge 2} \xi_h$ (measure how well $\CWD$ holds), which implies the problem instance with smaller $\xi$ has more regret and vice-versa. The value of $\xi$ is analogous to the minimum sub-optimality gap in the standard Multi-Armed Bandits setting. With a large context set, $\xi$ can be small, and its inverse relation in the regret captures the difficulty of the USS problem.

%% file: experiment.tex

We evaluate the performance of \ref{alg:CUSS_WD} on different problem instances derived from synthetic and real datasets. In our experiments, the data samples are treated as contexts. The labels of contexts are known but are never revealed to the algorithm. We use the labels to train classifiers offline that act as arms. Arm $i$ represents a logistic classifier with trained parameter $\theta_i$. A context (data sample) $x$ is assigned label $1$ from the $i$-th classifier with probability $\mu(x^\top\theta_i)$ and label $0$ with probability $1-\mu(x^\top\theta_i)$. The disagreement labels for $(i,j)$ pair is computed using the labels of classifier $i$ and $j$. To satisfy \cref{equ:disProb}, we use the polynomial kernel of degree two for mapping context into higher-dimensional space. Unlike other kernels, the polynomial kernel uses a well-defined feature map to lift the contexts into fixed, higher-dimensional space. 
The details of the used problem instances are as follows.

\textbf{Synthetic Dataset:} We consider $3$-dimensional synthetic dataset with $5000$ data samples. Each sample is represented by $x=(x_1, x_2, x_3)$, where the value of $x_j$ is drawn uniformly at random from $(-1, 1)$. A sample $x$ is labeled $0$ if the value of $(x_1 + x_1x_2 + x_3^2)$ is negative otherwise it is labeled $1$. We train five logistic classifiers on this synthetic dataset by varying the regularization parameter.  We then assign a positive cost to each classifier and order them by their increasing cost. We vary the cost of using classifiers to get different problem instances (see details in the supplementary material). 

\textbf{Real Datasets:} We applied our algorithm on PIMA Indian Diabetes \citep{UCI16_pima2016kaggale} dataset. Each sample has $8$ features related to the conditions of the patient. We split the features into three subsets and train a logistic classifier on each subset. We associate 1st classifier with the first $6$ features as input. These features include patient history/profile. The 2nd classifier, in addition to the $6$ features, utilizes the feature on the glucose tolerance test, and the 3rd classifier uses all the previous features and the feature that gives values of insulin test. Due to space constraints, the experiment results on Heart Disease dataset \citep{HEART98_robert1988va, UCI17_Dua2017} are given in the supplementary material.

\subsection{Experiments Results} 
We compare the performance of \ref{alg:CUSS_WD} on four problem instances derived from the synthetic dataset. The instances vary based on the cost of arms. All contexts in Instance $1$ do not satisfy $\WD$ property; hence it suffers linear regret as shown in \cref{fig:syn}. For the remaining instances, we set costs such that the value of $\xi$ increasing from Instance $2$ to $4$. As expected, the regret decreases from Instance $2$ to $4$, as seen in \cref{fig:syn}. We also compare \ref{alg:CUSS_WD} against an algorithm where the learner receives true labels as feedback. In particular, the learner knows whether the classifier's output is correct or not and can estimate their error rates. We implement this `supervised' setting by replacing disagreement probability in \cref{eq:selectDisProbHigh} with estimated error rates. As expected, the regret with supervision has lower than the \ref{alg:CUSS_WD} regret (unsupervised) in \cref{fig:supComp}. It is qualitatively interesting because these plots demonstrate that, in typical cases, our unsupervised algorithm can eventually learn to perform as good as an algorithm with knowledge of true labels. 

\begin{figure}[H]
	\captionsetup[subfigure]{justification=centering}
	\begin{subfigure}[b]{0.30\linewidth}
		\centering
		\includegraphics[scale=0.30]{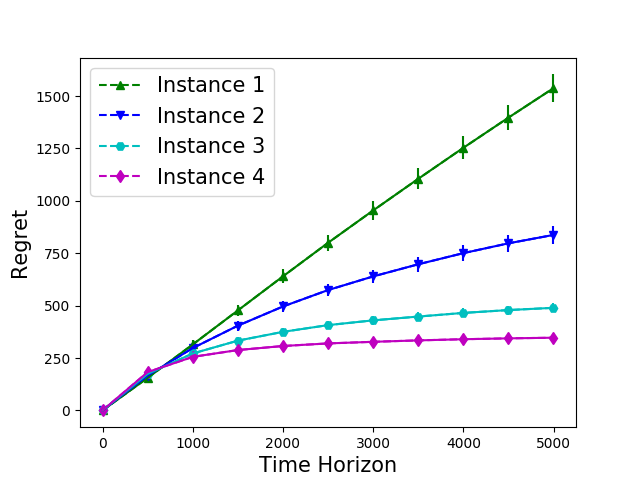}
		\caption{\small Synthetic dataset: Regret for different instance}
		\label{fig:syn}
	\end{subfigure}
	\quad
	\begin{subfigure}[b]{0.30\linewidth}
		\centering
		\includegraphics[scale=0.30]{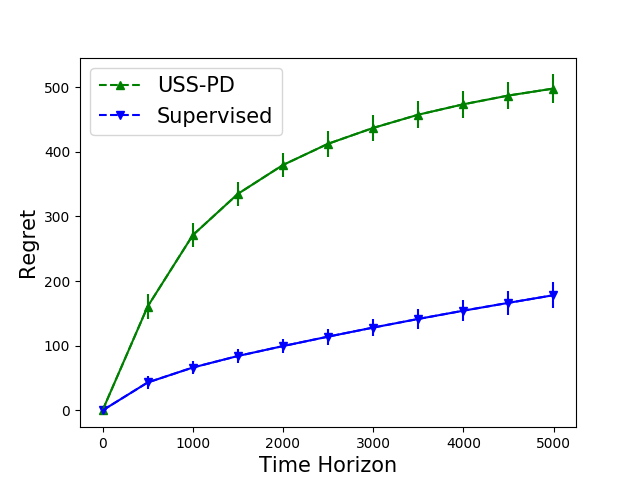}
		\caption{\small Supervised Setting (Instance 3 of Synthetic Dataset)}
		\label{fig:supComp}
	\end{subfigure}
	\quad
	\begin{subfigure}[b]{0.275\linewidth}
		\centering
		\includegraphics[scale=0.275]{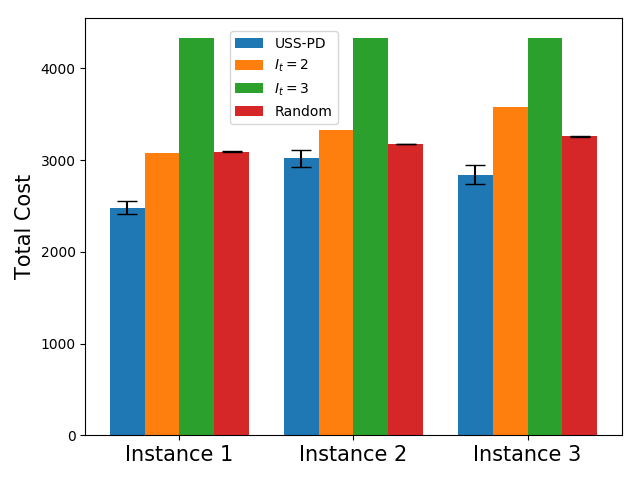}
		\caption{\small Total cost for PIMA Indian Diabetes dataset}
		\label{fig:diabetes}
	\end{subfigure}	
	\caption{\small Performance of \ref{alg:CUSS_WD} on different problem instances derived from synthetic and real datasets.}
	\label{fig:ussExp}
\end{figure}

We derive three problem instances from PIMA Indian Diabetes dataset by varying the costs of using classifies. Since all contexts of these problem instances do not satisfy $\WD$ property (see details in the supplementary material), we used cumulative total expected cost as a performance measure, where the cumulative total expected cost is given by $\sum_{t=1}^T(\gamma_{I_t}(x_t) + C_{I_t})$. We compare the performance of \ref{alg:CUSS_WD} with three baseline policies -- the first baseline policy uses the third classifier irrespective of contexts, and it is denoted as policy `$I_t=3$' (plays arm $3$ in each round). The second baseline policy uses the second classifier for all contexts, and it is denoted as policy `$I_t=2$'. The third baseline policy is `Random,' which selects an arm uniformly at random in each round. In all three problem instances, we observe that \ref{alg:CUSS_WD} performs better than the baselines, as shown in \cref{fig:diabetes}. 

We repeat each of the above experiments $100$ times, and then the average regret is presented with a $95$\% confidence interval. The vertical line on each plot shows the confidence interval.

%% file: appendix.tex

\section{Missing proofs from \cref{sec:problem_setting}}

\subsection{Proof of \cref{lem:err_prob_contx}} 
\ErrProbContx*
\begin{proof} 
	Using definition of $\gamma_i(\xt) \doteq \Prob{\Yti \neq \Yt|X=\xt}$, we get	
	\begin{align*}
		\gamma_i(\xt) - \gamma_j(\xt) &= \Prob{\Yti \neq \Yt|X=\xt} - \Prob{\Ytj \neq \Yt|X=\xt}. 
	\end{align*}
	As the observed feedback is binary, if $\Yti = \Ytj$ and $\Yti \ne \Yt$ then $\Ytj \ne \Yt$,
	\begin{align*}
		\gamma_i(\xt) - \gamma_j(\xt)&= \cancel{\Prob{\Yti \neq\Yt, \Yti = \Ytj|X=\xt}} + \Prob{\Yti \neq\Yt, \Yti \ne \Ytj|X=\xt} \\
		&- \cancel{\Prob{\Ytj \neq\Yt, \Yti = \Ytj|X=\xt}} - \Prob{\Ytj \neq\Yt, \Yti \ne \Ytj|X=\xt}.
	\end{align*}
	Adding and subtracting $\Prob{\Yti =\Yt, \Yti \ne \Ytj|X=\xt}$,
	\begin{align*}
		\gamma_i(\xt) - \gamma_j(\xt)&= \Prob{\Yti \neq\Yt, \Yti \ne \Ytj|X=\xt} + \Prob{\Yti =\Yt, \Yti \ne \Ytj|X=\xt} \\
		& - \Prob{\Ytj \neq\Yt, \Yti \ne \Ytj|X=\xt} - \Prob{\Yti =\Yt, \Yti \ne \Ytj|X=\xt}. 
	\end{align*}
	If $\Yti \ne \Ytj$ and $\Ytj \ne\Yt$ then $\Yti =\Yt$,
	\begin{align*}
		\gamma_i(\xt) - \gamma_j(\xt)&= \Prob{\Yti \ne \Ytj|X=\xt} - \Prob{\Yti =\Yt, \Yti \ne \Ytj|X=\xt}\\
		&\qquad - \Prob{\Yti =\Yt, \Yti \ne \Ytj|X=\xt} \\
		&= \Prob{\Yti \neq \Ytj|X=\xt} -2\Prob{\Yti =\Yt, \Ytj \ne \Yt|X=\xt}. \\
		\implies \gamma_i(x_t^i) -\gamma_j(x_t^j) & = \pijt -2\Prob{\Yti =\Yt, \Ytj \ne \Yt|X=\xt}. \qedhere
	\end{align*}
\end{proof}

\subsection{Proof of \cref{lem:Bx}} 
\SetBx*
\begin{proof}
	Let $i_t^\star$ be an optimal arm for a context $x_t$. As $\pijt =\mathbb{P}\{\Yti \ne\Ytj |X=\xt \}$ and $i_t^\star$ is an optimal arm, we have $\forall j<i_t^\star:\, C_{i_t^\star} - C_j \le \mathbb{P}\{\Yt^{i_t^\star}\ne\Yt^j |X=\xt \} \implies C_{i_t^\star} - C_j \ngtr \mathbb{P}\{\Yt^{i_t^\star}\ne\Yt^j |X=\xt \} \implies \forall j<i_t^\star \notin \Bt$.  If any sub-optimal arm $h \in \Bt$ then $h > i_t^\star$ i.e.,
	\eqs{
		\Bt = \{i_t^\star, h_1, \ldots, h_n, K\},
	} 	
	where $i_t^\star < h_1 < \cdots < h_n < K$.	By construction of set $\Bt$, the minimum indexed arm in set $\Bt$ is only the optimal arm.
\end{proof}

\subsection{Proof of \cref{thm:learnCWD}}
We need the following results to proof of \cref{thm:learnCWD}.
\begin{lem}
	\label{lem:xCostRange1}
	Let $i<j$ and $x_t \in \cX$ be any context. Assume 
	\begin{equation}
		\label{eqn:xCostRange1}
		C_j -C_i \notin \left (\gamma_i(x_t)-\gamma_j(x_t), \Prob{\Yti \ne \Ytj|X=\xt}\right].
	\end{equation}
	Then, $C_j-C_i >  \gamma_i(x_t)-\gamma_j(x_t) $ iff $C_j-C_i > \Prob{\Yti \ne \Ytj|X=\xt}$.
\end{lem}

\begin{proof}
	Assume that $C_j-C_i >  \gamma_i(x_t)-\gamma_j(x_t)$. As $C_j -C_i \notin \Big(\gamma_i(x_t)-\gamma_j(x_t),$ $\Prob{\Yti \ne \Ytj|X=\xt}\Big]$, we get $C_j-C_i > \Prob{\Yti \ne \Ytj|X=\xt}$.
	The proof of other direction follows by noting that  $ \Prob{\Yti \ne \Ytj|X=\xt}\geq  \gamma_i(x_t)-\gamma_j(x_t)$.
\end{proof}

\begin{lem}
	\label{lem:xCostRange2}
	Let $i>j$ and $x_t \in  \cX$ be any context. Assume 
	\begin{equation}
		\label{eqn:xCostRange2}
		C_i -C_j \notin \left (\gamma_j(x_t)-\gamma_i(x_t), \Prob{\Yti \ne \Ytj|X=\xt}\right ].
	\end{equation}
	Then, $C_i-C_j \leq \gamma_j(x_t)-\gamma_i(x_t) $ iff $C_j-C_i \leq \Prob{\Yti \ne \Ytj|X=\xt}$.
\end{lem}

\begin{proof}
	Let $C_i-C_j \leq  \gamma_j(x_t)-\gamma_i(x_t) $. As $\gamma_j(x_t)-\gamma_i(x_t) \leq \Prob{\Yti \ne \Ytj|X=\xt}$, we get $C_i-C_j \leq \Prob{\Yti \ne \Ytj|X=\xt}$.
	
	The condition $C_i-C_j \leq \Prob{\Yti \ne \Ytj|X=\xt}$ along with $C_i -C_j \notin \Big(\gamma_j(x_t)-\gamma_i(x_t),$ $\Prob{\Yti \ne \Ytj|X=\xt}\Big]$ implies the other direction, i.e., $C_i-C_j \leq  \gamma_j(x_t)-\gamma_i(x_t) $. 
\end{proof}

\begin{lem}
	\label{lem:xWD2} 
	Let  $\ist$ be an optimal arm for a context $x_t$. Any problem instance $P \in \USS$ is learnable if for every context in $P$ following holds:
	\begin{equation*}
		\forall j>\ist,\; C_j -C_\ist >\Prob{\Yt^\ist \ne \Ytj|X=\xt}.
	\end{equation*}
\end{lem}
The proof of \cref{lem:xWD2} follows from \cref{lem:xCostRange1} and \cref{lem:xCostRange2}. 
Now we give proof for Theorem \ref{thm:learnCWD}.

\learnCWD*
\begin{proof}
	Let $\ist$ be an optimal arm for a context $x_t$. It is enough to prove that any problem instance $P \in \USS$  is learnable if
	\begin{align*}
		\forall j>\ist,\; C_j -C_\ist >\Prob{\Yt^\ist \ne \Ytj|X=\xt}. \hspace{4mm}\text{(definition of $\CWD$ property)}
	\end{align*} 
	From \cref{lem:xCostRange1} and \cref{lem:xCostRange2}, if the optimal arm satisfies following conditions,
	\begin{align*}
		&\forall j>\ist, C_j -C_\ist \notin \left (\gamma_\ist(x_t)-\gamma_j(x_t), \Prob{\Yt^\ist \ne \Ytj|X=\xt}\right] \text{ and}\\
		&\forall j<\ist, C_{i^\star} -C_j \notin \left (\gamma_j(x_t)-\gamma_\ist(x_t), \Prob{\Yt^\ist \ne \Ytj|X=\xt}\right ],
	\end{align*}
	
	then, for $j >\ist, C_j -C_\ist>\gamma_\ist(x_t) - \gamma_j(x)$	iff 	$C_j -C_\ist >\Prob{\Yt^\ist \ne \Ytj|X=\xt}$ and for  $j <\ist, C_{i^\star}-C_j\leq  \gamma_{j}(x) -\gamma_\ist(x_t)$	iff 	$C_j -C_\ist \leq \Prob{\Yt^\ist \ne \Ytj|X=\xt}$. Hence we can use $\Prob{\Yti \ne \Ytj|X=\xt}$ as a proxy for $\gamma_\ist(x) -\gamma_j(x)$ to make decision about the optimal arm.
	Now notice that for $j <\ist$, $C_{i^\star}-C_j \leq \gamma_{j}(x) -\gamma_\ist(x_t)$. Hence, 
	\begin{align*}
		& \forall j<\ist, C_{i^\star} -C_j \notin \left (\gamma_j(x_t)-\gamma_\ist(x_t), \Prob{\Yt^\ist \ne \Ytj|X=\xt}\right] \text{ and} \\
		& \forall j>\ist, C_j -C_\ist \notin \left (\gamma_\ist(x_t)-\gamma_j(x_t), \Prob{\Yt^\ist \ne \Ytj|X=\xt}\right] \numberthis \label{equ:CostNotin}
	\end{align*}
	are sufficient for learnability. Note that \cref{equ:CostNotin} is equivalent to 
	\begin{equation}
		\label{equ:CWD_Thm1}
		\forall j>\ist,\; C_j -C_\ist >\Prob{\Yt^\ist \ne \Ytj|X=\xt}. 
	\end{equation}
	Note that if \cref{equ:CWD_Thm1} does not hold, then knowing $\Prob{\Yt^\ist \ne \Ytj|X=\xt}$ is not sufficient for finding the optimal arm.
\end{proof}

\subsection{Regret  decomposition when contexts satisfy $\WD$ with some known probability}
Without knowing the disagreement probability, it is impossible to check whether a context satisfies $\WD$ property or not. Hence we consider a case where a context can satisfy $\WD$ property with some fixed probability. For such cases, we can decompose the regret into two parts: regret due to the contexts that satisfy $\WD$ property and regret due to the contexts that do not satisfy $\WD$ property. Note that the regret can be linear due to the contexts that do not satisfy the $\WD$ condition. 

Our next result gives the upper bound on the regret where the contexts satisfy $\WD$ property with a known fixed probability.
\begin{lem}
	Let $\rho$ be the probability of context that it does not satisfy the $\WD$ property and $R_{max}$ be the maximum regret incurred for any context. If $\Regret_T$ is the regret incurred when all contexts satisfy $\WD$ property then, the regret incurred when contexts satisfy $\WD$ with probability $(1-\rho)$ is given by
	\begin{equation*}
	\Regret_T^\prime \le (1-\rho)\Regret_T + \rho R_{max}T.
	\end{equation*}
\end{lem}
\begin{proof}
	Let $\rho$ be the probability of context that it does not satisfy the $\WD$ property and $r_t(I_t, i^\star_t)$ be the regret incurred for selecting sub-optimal arm $I_t$ for the context $x_t$. Then the regret can be decomposed into two parts as follows:
	\begin{align*}
		\Regret_T^\prime &= \EE{\sum_{t=1}^T \left[\one{x_t \mbox{ satisfies } \WD} r_t(I_t, i^\star_t) + \one{x_t \mbox{ does not satisfy } \WD} r_t(I_t, i^\star_t) \right]}\\
		&= \EE{\sum_{t=1}^T \one{x_t \mbox{ satisfies } \WD} r_t(I_t, i^\star_t)} + \EE{\sum_{t=1}^T \one{x_t \mbox{ does not satisfy } \WD} r_t(I_t, i^\star_t)}\\
		&= \sum_{t=1}^T \Prob{x_t \mbox{ satisfies } \WD} r_t(I_t, i^\star_t) + \sum_{t=1}^T\Prob{x_t \mbox{ does not satisfy } \WD} r_t(I_t, i^\star_t) \numberthis \label{eqn:probRegretUB}.
	\end{align*}
	
	First, we will bound the regret due to the contexts that do not satisfy $\WD$ property (second term of \cref{eqn:probRegretUB}). Note that the context that does not satisfy $\WD$ property, the learner can not make the correct decision hence always incurs regret. Since the maximum regret is upper bounded by $R_{max}$, we have	
	\begin{align*}
		\sum_{t=1}^T\Prob{x_t \mbox{ does not satisfy } \WD} r_t(I_t, i^\star_t) &\le \sum_{t=1}^T\Prob{x_t \mbox{ does not satisfy } \WD} R_{max}
		\intertext{Since $\rho$ is the probability of context that it does not satisfy the $\WD$ property, we get }
		&= \sum_{t=1}^T \rho R_{max}\\
		\implies \sum_{t=1}^T\Prob{x_t \mbox{ does not satisfy } \WD} r_t(I_t, i^\star_t) &\le \rho R_{max}T. \numberthis \label{eqn:probRegretNoCWD}
	\end{align*}
	
	Now we will bound the regret due to the contexts which satisfy $\WD$ property (first term in \cref{eqn:probRegretUB}). Since any context satisfies $\WD$ with $1-\rho$ probability, we have
	\begin{align}
	\label{eqn:probRegretCWD}
	\sum_{t=1}^T \Prob{x_t \mbox{ satisfies } \WD} r_t(I_t, i^\star_t) = \sum_{t=1}^T (1-\rho)  r_t(I_t, i^\star_t) = (1-\rho)  \sum_{t=1}^T r_t(I_t, i^\star_t).
	\end{align}
	
	By assuming that all contexts are satisfying $\WD$ property, we have regret $\Regret_T = \sum_{t=1}^T  r_t(I_t, i^\star_t)$. Using it with \cref{eqn:probRegretNoCWD} in \cref{eqn:probRegretCWD}, we get
	\begin{equation*}
		\Regret_T^\prime \le (1-\rho)\Regret_T + \rho R_{max}T. \qedhere
	\end{equation*}
\end{proof}

\section{Missing proofs from \cref{sec:algorithm}}
\subsection{Proof of \cref{lem:min_eigen_lb}}
\minEigenLB*
\begin{proof}
	The result is adapted from \cite[Proposition 1]{ICML17_li2017provably}, which uses the standard random matrix theory result from \cite[Theorem 5.39]{Book_vershynin_2012}. We need to carefully construct the sample complexity bound for our case as the observations are only observed for a pair of arms.
\end{proof}

The following result is needed to prove \cref{lem:theta_est_glm}.
\begin{restatable}{lem}{detVt}
	\label{lem:detVt} 
	Let $\oVt = \lambda \mathrm{I}_{d^\prime} + a\sum_{s\in S_{ij}^t}\Phi_{ij}(x_s)\Phi_{ij}(x_s)^\top$  for any $(i,j)$ pair of arms, and $n_{ij}^t = |S_{ij}^t|$. Then 
	\begin{equation*}
		det(\oVt) \le  \left(\lambda + an_{ij}^t/d^\prime\right)^{d^\prime}.
	\end{equation*}
\end{restatable}
\begin{proof} 
	The proof is adapted from Lemma 10 of \cite{NIPS11_abbasi2011improved}. By using inequality of arithmetic and geometric means, we have $det(\oVt) \le (trace(\oVt)/{d^\prime})^{d^\prime}$. As the trace of matrix is a linear mapping i.e. $trace(A+B) = trace(A) + trace(B)$, hence, we get
	\begin{align*}
		trace(\oVt) &= trace(\lambda I_{d^\prime}) + a\sum_{s\in S_{ij}^t}trace\left(\Phi_{ij}(x_s)\Phi_{ij}(x_s)^\top\right) \\
		&= \lambda{d^\prime} + a\sum_{s\in S_{ij}^t}\norm{\Phi_{ij}(x)}_2^2 \\
		&\le \lambda{d^\prime} + an_{ij}^t. \hspace{5mm}\left(\text{as $\norm{\Phi_{ij}(x)}_2\le 1$ and $n_{ij}^t = |S_{ij}^t|$}\right)
	\end{align*}
	Using upper bound of $trace(\oVt)$ for bounding $det(\oVt)$, we get
	\begin{equation*}
		det(\oVt) \le (trace(\oVt)/{d^\prime})^{d^\prime} \le \left((\lambda{d^\prime} + an_{ij}^t)/{d^\prime}\right)^{d^\prime}  \le \left(\lambda + an_{ij}^t/{d^\prime}\right)^{d^\prime}. \qedhere
	\end{equation*}
\end{proof}

\subsection{Proof of \cref{lem:theta_est_glm}}
\thetaEstGLM*
\begin{proof}
	Let $\oVt = \lambda\mathrm{I}_{d^\prime} + \Vt$. If \cref{equ:glm_est} is used for estimation of unknown parameter $\Ts$ then by using Eq. (26) and Lemma 8 of \cite{ICML17_li2017provably} with $\lambda_{min}(V_{ij}^{m+1}) \ge 1$, we have
	\begin{align*}
		\norm{\Te - \Ts}_{\Vt} \le \frac{1}{\kappa}\norm{\sum_{s\in S_{ij}^t} \epsilon_s \Phi_{ij}(x_s)}_{(\Vt)^{-1}}
		\le \frac{(1-\lambda)^{\frac{-1}{2}}}{\kappa}\norm{\sum_{s\in S_{ij}^t} \epsilon_s \Phi_{ij}(x_s)}_{(\oVt)^{-1}}.
	\end{align*}
	Using upper bound of $\norm{\sum_{s\in S_{ij}^t}\epsilon_s\Phi_{ij}(x_s)}_{(\oVt)^{-1}}$ as given in Theorem 1 of \cite{NIPS11_abbasi2011improved} where $\epsilon_s$ is $\sigma-$subGaussian random variable, the following inequality holds with at least probability $1-2\delta/K^2$	
	\begin{align*}
		&\le   \frac{(1-\lambda)^{\frac{-1}{2}}}{\kappa}\sqrt{2\sigma^2\log\left(\frac{det(\oVt)^{1/2}det(\lambda\mathrm{I}_{d^\prime})^{-1/2}}{2\delta/K^2}\right)}\\ &=\frac{\sigma(1-\lambda)^{\frac{-1}{2}}}{\kappa}\sqrt{2\log\left(\frac{det(\oVt)}{det(\lambda\mathrm{I}_{d^\prime})}\right)^{\frac{1}{2}} + 2\log\left(\frac{K^2}{2\delta}\right)}.
	\end{align*}
	Upper bounding $det(\oVt)$ with $a = 1$, $\lambda=1/2$, and $n_{ij}^t \le t$ by using \cref{lem:detVt}, we get
	\begin{align*}
		\implies \norm{\Te - \Ts}_{\Vt} &\le\frac{2\sigma}{\kappa}\sqrt{\frac{d^\prime}{2}\log\left(1 + \frac{2n_{ij}^t}{d^\prime} \right) + \log\left(\frac{K^2}{2\delta}\right)} \\
		&\le \frac{2\sigma}{\kappa}\sqrt{\frac{d^\prime}{2}\log\left(1 + \frac{2t}{d^\prime} \right) + \log\left(\frac{K^2}{2\delta}\right)}.  \qedhere 
	\end{align*}
\end{proof}

\subsection{Proof of \cref{lem:subOptimalLowerSelection}}
\subOptimalLowerSelection*
\begin{proof}
	If sub-optimal arm $l<i^\star_t$ is preferred by \ref{alg:CUSS_WD} then using Eq. \eqref{def_prefer_h}, we get
	\begin{align*}
		\one{l \succ_t i^\star_t, i^\star_t= i} &= \one{C_i - C_l > \tplit, I_t = l, i^\star_t= i} \\
		&\le \one{C_i - C_l > \tplit}. \hspace{5mm}\text{(as $A\cap B \cap C \subseteq A$)}
	\end{align*}
	
	Using $C_i - C_l = p_{li}(x_t) - \xi_{li}(x_t)$ for  $l< i$, we have
	\begin{align*}
		\implies \one{l \succ_t i^\star_t, i^\star_t= i} &= \one{p_{li}(x_t) - \xi_{li}(x_t) > \tplit} = \one{p_{li}(x_t) - \tplit >  \xi_{li}(x_t)}. 
	\end{align*}
	
	Using definition of $p_{li}(x_t)$ and $\tplit$,
	\begin{align*}
		\implies \one{l \succ_t i^\star_t, i^\star_t = i} = \one{\mu(\Phi_{li}(x_t)^\top \Tsl) - \mu\left(\Phi_{li}(x_t)^\top \Tel + \alpha_{li}^t\norm{ \Phi_{li}(x_t)}_{\Vinvl}\right) > \xi_{li}(x_t)}. 
	\end{align*}
	
	Since $\mu(\cdot)$ is an increasing function and using $\alpha_{li}^t$ as defined in \cref{lem:theta_est_glm}, $\mu\left(\Phi_{li}(x_t)^\top \Tel + \alpha_{li}^t\norm{ \Phi_{li}(x_t)}_{\Vinvl}\right)$ is the upper bound on $\mu(\Phi_{li}(x_t)^\top \Tsl)$ for all $(l,i)$ pairs with probability at least $1-\delta/2$. We show it as follows:
	\begin{align*}
		\Phi_{li}(x_t)^\top \Tsl &= \Phi_{li}(x_t)^\top \Tel + \Phi_{li}(x_t)^\top (\Tsl - \Tel)\\
		&= \Phi_{li}(x_t)^\top \Tel + \norm{\Phi_{li}(x_t)}_{\Vinvl}\norm{\Tsl - \Tel}_{\Vtl}\\
		\implies \Phi_{li}(x_t)^\top \Tsl &\le \Phi_{li}(x_t)^\top \Tel + \alpha_{li}^t\norm{ \Phi_{li}(x_t)}_{\Vinvl}. &\hspace{-1.5cm}\left(\text{using }\norm{\Tsl - \Tel}_{\Vtl} \le \alpha_{li}^t\right) \\
		\intertext{Since $\mu(\cdot)$ is an increasing function,}
		\implies \mu(\Phi_{li}(x_t)^\top \Tsl) &\le \mu\left(\Phi_{li}(x_t)^\top \Tel + \alpha_{li}^t\norm{ \Phi_{li}(x_t)}_{\Vinvl}\right). 
	\end{align*}

	Hence, any sub-optimal arm smaller than the optimal arm is selected by \ref{alg:CUSS_WD} with probability at most $\delta/2$. It completes the proof of the lemma.
\end{proof}

\subsection{Proof of \cref{lem:subOptimalUpperSelection}}
\subOptimalUpperSelection*
\begin{proof}
	If sub-optimal arm $h>i^\star_t$ is preferred by \ref{alg:CUSS_WD} then using Eq. \eqref{def_prefer_l}, we get	
	\begin{align*}
		\one{h \succ_t i, i^\star_t= i} &= \one{C_h - C_i < \tpiht, h \succ_t i^\star_t, i^\star_t= i} \\ & \le \one{C_h - C_i < \tpiht}. \hspace{5mm}\text{(as $A\cap B \cap C \subseteq A$)}
	\end{align*}
	
	Using $C_h - C_i = p_{ih}(x_t) + \xi_{ih}(x_t)$ for $h > i$, we get
	\begin{align*}
		\implies \one{h \succ_t i, i^\star_t= i} &= \one{p_{ih}(x_t) + \xi_{ih}(x_t) < \tpiht} = \one{\tpiht - p_{ih}(x_t) >  \xi_{ih}(x_t)}. 
	\end{align*}
	
	Using definition of $p_{ih}(x_t)$ and $\tpiht$, 
	\begin{align*}
		\implies \one{h \succ_t i, i^\star_t= i}&= \one{\mu\left(\Phi_{ih}(x_t)^\top \Tsh + \alpha_{ih}^t\norm{ \Phi_{ih}(x_t)}_{\Vinvh}\right) - \mu(\Phi_{ih}(x_t)^\top \Teh)  > \xi_{ih}(x_t)}. 
	\end{align*}
	
	As $\mu$ is Lipschitz, $|\mu(z_1) - \mu(z_2)| \le k_\mu|z_1 - z_2|$ where $k_\mu$ is Lipschitz constant, we have
	\begin{align*}
		&\le \one{k_\mu|\Phi_{ih}(x_t)^\top \Tsh + \alpha_{ih}^t\norm{ \Phi_{ih}(x_t)}_{\Vinvh} - \Phi_{ih}(x_t)^\top \Teh| > \xi_{ih}(x_t)} \\
		&\le \one{k_\mu|\Phi_{ih}(x_t)^\top \Tsh - \Phi_{ih}(x_t)^\top \Teh| + k_\mu\alpha_{ih}^t\norm{ \Phi_{ih}(x_t)}_{\Vinvh}> \xi_{ih}(x_t)} \\
		&= \one{k_\mu|\Phi_{ih}(x_t)^\top (\Tsh - \Teh)| + k_\mu\alpha_{ih}^t\norm{ \Phi_{ih}(x_t)}_{\Vinvh} > \xi_{ih}(x_t)}.
	\end{align*}
	
	Using Cauchy-Schwartz inequality and $\norm{ x}_A^2 = x^\top Ax$, we get
	\begin{align*}
		&\le \one{k_\mu\norm{\Phi_{ih}}_{\Vinvh} \norm{\Tsh - \Teh}_{\Vth} + k_\mu\alpha_{ih}^t\norm{ \Phi_{ih}(x_t)}_{\Vinvh} > \xi_{ih}(x_t)}.
	\end{align*}
	
	As $\norm{\Tsh- \Teh}_{\Vth} \le \alpha_{ih}^t$, we get
	\begin{align*}
		&\le \one{k_\mu\alpha_{ih}^t\norm{ \Phi_{ih}(x_t)}_{\Vinvh}  + k_\mu\alpha_{ih}^t\norm{ \Phi_{ih}(x_t)}_{\Vinvh} > \xi_{ih}(x_t)}\\
		&= \one{2k_\mu\alpha_{ih}^t\norm{ \Phi_{ih}(x_t)}_{\Vinvh} > \xi_{ih}(x_t)}. 
	\end{align*}

	As $\norm{\Phi_{ih}(x_t)}_{\Vinvh} \le \norm{\Phi_{ih}(x_t)}_2/\sqrt{\lambda_{min}(\Vth)}$ where $\lambda_{min}(\Vth)$ is the smallest eigenvalue of matrix $\Vth$ and $\norm{\Phi_{ih}(x_t)}_2 \le 1$, we get
	\begin{align*}
		\implies \one{h \succ_t i, i^\star_t= i}&\le \one{2k_\mu\alpha_{ih}^t > \xi_{ih}(x_t)\sqrt{\lambda_{min}(\Vth)}}.  \label{equ:IndHighSelCond} \numberthis
	\end{align*}
	The event on LHS is subset of event of RHS in \cref{equ:IndHighSelCond}. By changing $i$ to $i^\star_t$ completes the proof of the lemma.
\end{proof}

\subsection{Proof of \cref{thm:cum_reg_glm}}
\cumRegGLM*
\begin{proof}
	The regret for $T$ rounds in the Contextual USS problem is given by
	\begin{align*}
		\Regret_T &= \sum_{t=1}^T\left(C_{I_t} + \gamma_{I_t}(x_t) - (C_{i^\star_t} +\gamma_{i^\star_t}(x_t))\right).
	\end{align*}
	
	As $R_{max}$ denote the maximum regret incurred for any context, we get
	\begin{align}
		\Regret_T &\le R_{max}\sum_{t=1}^T \one{I_t \ne i^\star_t}. \label{equ:RegretInd}
	\end{align}
	
	As $\one{I_t \ne i^\star_t}$ has two random quantities $I_t$ and $i^\star_t$, we can re-write it as follows:
	\begin{align*}
		\one{I_t \ne i^\star_t} &= \sum_{l < i}\one{I_t = l, i^\star_t= i} + \sum_{h^\prime>i}\one{I_t = h^\prime, i^\star_t=i}. 
	\end{align*}
	
	Note that if \ref{alg:CUSS_WD} selects $l < i^\star_t$ then $l$ must be preferred over $i^\star_t$ whereas if $h^\prime > i^\star_t$ is selected then there exists an arm $h > i^\star_t$ which is preferred over $i^\star_t$. Hence, we have
	\begin{align} 
		\label{equ:IndWrongChoice}
		\one{I_t \ne i^\star_t}  &= \sum_{l < i}\one{l \succ_t i^\star_t, i^\star_t= i} + \sum_{h^\prime>i}\one{I_t = h^\prime, h \succ_t \ist, i^\star_t=i}\nonumber \\
		&\le \sum_{l < i}\one{l \succ_t i^\star_t, i^\star_t= i} + \sum_{h>i}\one{h \succ_t \ist, i^\star_t=i}.
	\end{align}
	
	Using above bound in \cref{equ:RegretInd}, we get
	\begin{align*}
		\Regret_T &\le R_{max}\sum_{t=1}^T \left[\sum_{l < i}\one{l \succ_t i^\star_t, i^\star_t= i} + \sum_{h>i}\one{h\succ_t i^\star_t, i^\star_t=i}\right].
	\end{align*}

	From \cref{lem:subOptimalLowerSelection}, $\one{l \succ_t i^\star_t, i^\star_t= i} = 0$ for any $l<i$ with probability at least $1-\delta/2$, then the regret becomes
	\begin{align*}
		\Regret_T &\le R_{max}\sum_{t=1}^T\sum_{h>i}\one{h \succ_t i^\star_t, i^\star_t=i} = R_{max} \sum_{h>i}\sum_{t=1}^T \one{h \succ_t i^\star_t, i^\star_t=i} \le R_{max} \sum_{h=2}^K\sum_{t=1}^T \one{h \succ_t i^\star_t, i^\star_t<h}.
	\end{align*}
	
	Note that $\alpha_{ih}^t$ is slowly increasing value with $t$ that implies $\alpha_{ih}^t\le \alpha_{ih}^T$ for all $t \le T$.
	Using \cref{lem:min_eigen_lb} with $\Psi=  \left(\frac{2k_\mu \alpha_{ih}^T}{\xi_{ih}}\right)^2$, $\Sigma_{ih} = \EE{\Phi_{ih}(X_s)\Phi_{ih}(X_s)^T}$ where $s \in S_{ih}^t$, after
	\begin{equation*}
		n_{ih}^T \doteq \left(\frac{C_1\sqrt{d^\prime}+C_2\sqrt{\log (K^2/2\delta)}}{\lambda_{min}(\Sigma_{ih})}\right)^2+ \frac{2}{\lambda_{min}(\Sigma_{ih})}\left(\frac{2k_\mu\alpha_{ih}^T}{\xi_{ih}}\right)^2
	\end{equation*}
	observations  for arm pair $(i,h)$ the $\lambda_{min}(\Vth) \ge \left(\frac{2k_\mu\alpha_{ih}^T}{\xi_{ih}}\right)^2$ with probability at least $1-2\delta/K^2$. Therefore, after having $n_{ih}^T$ observations, the sub-optimal arm $h(>i)$ will not be preferred over optimal arm $i$ with probability at least $1-2\delta/K^2$. Therefore, with probability at least $1-2\delta/K^2$, following equations also hold
	\begin{align*}
		&\one{I_t = h, i^\star_t= i, |S_{ih}^t| \ge n_{ih}^T} = 0
		\implies \sum_{t=1}^{T}\one{I_t = h, i^\star_t= i, h>i} \le n_{ih}^T.
	\end{align*}

	Due to the problem structure, whenever an arm $h$ is selected, disagreement labels for all arm pair $(i,j)$ where $i<j\le h$ are observed. Therefore, with probability at least $1-2\delta/K$ (by union bound), the maximum number of times an arm $h$ is selected when the optimal arm's index is smaller than $h$ is $n_h^T$ such that 
	\begin{align*}
    	n_{h}^T &= \left(\frac{C_1\sqrt{d^\prime}+C_2\sqrt{\log (K^2/2\delta)}}{\lambda_{\Sigma}}\right)^2+ \frac{2}{\lambda_{\Sigma}}\left(\frac{2k_\mu\alpha_T}{\xi_{h}}\right)^2 \\
    	&= \left(\frac{C_1\sqrt{d^\prime}+C_2\sqrt{\log (K^2/2\delta)}}{\lambda_{\Sigma}}\right)^2+ \frac{8}{\lambda_{\Sigma}}\left(\frac{k_\mu\alpha_T}{\xi_{h}}\right)^2 
	\end{align*}
	where $\xi_h = \min\limits_{i<h,t \ge 1} \xi_{ih}(x_t)$, $\lambda_{\Sigma} = \min\limits_{i < j\le K}\lambda_{min}\left(\EE{\Phi_{ij}(X_s)\Phi_{ij}(X_s)^\top}\right)$ and $\alpha_T \ge \max\limits_{i<h}\alpha_{ih}^T$. By using union bound, we get following bound with probability at least $1-\delta/2K$
	\begin{align*}
		\sum_{t=1}^{T}\one{I_t = h, i^\star_t < h} \le n_{h}^T. \numberthis \label{equ:hIndBnd}
	\end{align*}
	
	From \cref{equ:hIndBnd}, using $\sum_{t=1}^T \one{I_t = h, i^\star_t<h} \le n_{h}^T$ and value of $n_h^T$, we get following upper bound on regret that holds with probability at least $1-\delta$ by union bound
	\begin{align*}
		\Regret_T \le R_{max}\sum_{h=2}^{K} n_{h}^T =  R_{max}\sum_{h=2}^{K} \left( \left(\frac{C_1\sqrt{d^\prime}+C_2\sqrt{\log (K^2/2\delta)}}{\lambda_{\Sigma}}\right)^2+ \frac{8}{\lambda_{\Sigma}}\left(\frac{k_\mu\alpha_T}{\xi_{h}}\right)^2\right).
	\end{align*}
	
	Using $\alpha_{T}= \frac{2\sigma}{\kappa}\sqrt{\frac{d^\prime}{2}\log\left(1 + 2T/{d^\prime} \right) + \log\left({K^2}/{2\delta}\right)}$ from Lemma \ref{lem:theta_est_glm} that ensures parameter $\Ts$ bounds for all pairs $(i,j)$ holds with probability at least $1 - \delta/2K$ (by union bound) for $T>m$ where $m=C\lambda_\Sigma^{-2}\left(d+\log(k^2/2\delta)\right) + 2\lambda_\Sigma^{-1}$ such that $\lambda_{min}(V_{ij}^{m+1}) \ge 1$ for all pair $(i,j)$, we have
	\begin{align*}
		\Regret_T &\le R_{max}\left(m + \sum_{h=2}^{K} n_{h}^T\right)\\
		\implies \Regret_T  &\le R_{max}\Bigg[m + \sum_{h=2}^{K} \Bigg(\hspace{-1mm} \Bigg(\frac{C_1\sqrt{d^\prime}+C_2\sqrt{\log \left(\frac{K^2}{2\delta}\right)}}{\lambda_{\Sigma}}\Bigg)^2  \\
		&\qquad + \frac{16}{\lambda_{\Sigma}}\left(\frac{k_\mu\sigma}{\xi_{h}\kappa}\right)^2  \left(\frac{d^\prime}{2}\log\left(1 + \frac{2T}{d^\prime} \right) + \log\left(\frac{K^2}{2\delta}\right)\right)\Bigg)\Bigg]. \qedhere
	\end{align*}
\end{proof}

%% file: lambda_algorithm.tex

\ref{alg:CUSS_WD} uses forced exploration by selecting arm $K$ until the correlation matrix $V_{ij}^t$ is not invertible for all $(i,j)$ pairs of arms. Further, the minimum eigenvalue of $V_{ij}^t$ for all $(i,j)$ pairs is needed to be larger than $1$ so that bound given in \cref{lem:theta_est_glm} holds. Alternatively, $V_{ij}^t$ can be initialized by adding a regularization term \citep{NIPS11_abbasi2011improved,ICML16_zhang2016online,NIPS17_jun2017scalable} to avoid forced exploration and then apply OFUL type analysis. We have given an algorithm named \ref{alg:CUSS_GLM2} which uses regularization term $\lambda\mathrm{I}_{d^\prime}$. However, its analysis still needed the minimum eigenvalue of the non-regularized part of the correlation matrix to become larger than some positive value (depends on $\lambda$ value), as shown in our next result.
\begin{algorithm}[H]
    \floatname{algorithm}{}
    \renewcommand{\thealgorithm}{USS-PD-$\lambda\mathrm{I}$}
	\caption{Algorithm for Contextual USS using Pairwise Disagreement with $\lambda\mathrm{I}$ Initialization}
	\label{alg:CUSS_GLM2}
	\begin{algorithmic}[1]
		\State \textbf{Input:} Tuning parameters: $\delta \in (0,1)$ and $\lambda>0$ 
		\State Select arm $K$ for first context $x_1$
		\State $\forall i< j \le K:$ set $\overline{V}_{ij}^{1} \leftarrow \lambda\mathrm{I}_{d^\prime} + \Phi_{ij}(x_{1}) {\Phi_{ij}(x_{1})}^\top$ and update $\hat\theta_{ij}^1$ by solving \cref{equ:glm_est} 
		\For{$t= 2, 3, \ldots$}
			\State Receive context $x_t$. Set $i=1$ and $I_t=0$
			\Do
				\State Play arm $i$ 
				\State $\forall j \in [i+1, K]:$ compute $\tpijt \leftarrow \mu\left(\Phi_{ij}(\xt)^\top \hat\theta_{ij}^{t-1} + \alpha_{ij}^{t-1}\norm{ \Phi_{ij}(\xt)}_{\left(\overline{V}_{ij}^{t-1}\right)^{-1}} \right)$
				\State If $\forall j \in [i+1, K]: C_j - C_i > \tpijt$ or $i=K$ then set $I_t = i$ else  set $i=i+1$
			\doWhile{$I_t=0$}
			\State Select arm $I_t$ and observe $Y_t^1, Y_t^2, \dots, Y_t^{I_t}$
			\State $\forall i< j \le I_t:$  update $\oVijt \leftarrow \overline{V}_{ij}^{t-1} + \Phi_{ij}(x_{t}) {\Phi_{ij}(x_{t})}^\top$ and $\Te$ by solving \cref{equ:glm_est} 
		\EndFor
	\end{algorithmic}
\end{algorithm}

\begin{restatable}{lem}{thetaEstGLM2}
	\label{lem:theta_est_glm2} 
	Let $\oVt = \lambda\mathrm{I}_{d^\prime} + \Vt$ for any $\lambda>0$ and $\norm{\theta_{ij}}_2 \le S$ for all $(i,j)$ pair. Then for any $t > \min\{s: \forall i<j \ni \lambda_{\min}(V_{ij}(s)) \ge 2\lambda\}$, the following event holds for \ref{alg:CUSS_GLM2} with probability at least $1-2\delta/K^2$,
	\begin{align*}
		&\norm{ \Te - \Ts }_{\oVt} \le \beta_{ij}^t,
	\end{align*}
	where $\beta_{ij}^t = \frac{2\sigma}{\kappa}\sqrt{\frac{d^\prime}{2}\log\left(1 + \frac{n_{ij}^t}{d^\prime\lambda} \right) + \log\left(\frac{K^2}{2\delta}\right)} + 2\lambda^{1/2}S $.
\end{restatable}
\begin{proof}
	By using $\norm{Z}_{A+B} \le \norm{Z}_{A} + \norm{Z}_{B}$ , we have
	\begin{align*}
		\norm{\Te - \Ts}_{\oVt} &\le \norm{\Te - \Ts}_{\Vt} + \norm{\Te - \Ts}_{\lambda\mathrm{I}_{d^\prime}}. & \hspace{-2cm}(\mbox{as } \oVt = \lambda\mathrm{I}_{d^\prime} + \Vt) \\
		\intertext{When \cref{equ:glm_est} is used for estimation of unknown parameter $\Ts$ then by using Eq. (26) and Eq. (27) of Lemma 8 of \cite{ICML17_li2017provably}, we have}
		& \le \frac{1}{\kappa}\norm{\sum_{s\in S_{ij}^t} \epsilon_s \Phi_{ij}(x_s)}_{\Vtinv} + 2\lambda^{1/2}S. &\hspace{-5cm} (\mbox{as } \norm{\theta_{ij}}_2 \le S) 
	\end{align*}
	
	Sherman Morrison formula gives $\norm{Z}_{\Vtinv} \le \left(1-\frac{\lambda}{\lambda_{\min}(\Vt)}\right)^{-\frac{1}{2}}\norm{Z}_{\oVtinv}$. Using it, we have
	\begin{align*}
		\norm{\Te - \Ts}_{\oVt} \le \frac{\left(1-\frac{\lambda} {\lambda_{\min}(\Vt)}\right)^{-\frac{1}{2}}}{\kappa}\norm{\sum_{s\in S_{ij}^t} \epsilon_s \Phi_{ij}(x_s)}_{\oVtinv} + 2\lambda^{1/2}S.
	\end{align*}
	
	Using upper bound on $\norm{\sum_{s\in S_{ij}^t}\epsilon_s\Phi_{ij}(x_s)}_{(\overline\Vt)^{-1}}$ as given in Theorem 1 of \cite{NIPS11_abbasi2011improved}, where $\epsilon_s$ is $\sigma-$subGaussian random variable and holds with probability at least $1-2\delta/K^2$, we get
	\begin{align*}	
		\norm{\Te - \Ts}_{\oVt} &\le \frac{\left(1-\frac{\lambda}{\lambda_{\min}(\Vt)}\right)^{-\frac{1}{2}}}{\kappa}\sqrt{2\sigma^2\log\left(\frac{det(\oVt)^{1/2}det(\lambda\mathrm{I}_{d^\prime})^{-1/2}}{2\delta/K^2}\right)}  + 2\lambda^{1/2}S \\
		&=\frac{\sigma\left(1-\frac{\lambda}{\lambda_{\min}(\Vt)}\right)^{-\frac{1}{2}}}{\kappa}\sqrt{2\log\left(\frac{det(\oVt)}{det(\lambda\mathrm{I}_{d^\prime})}\right)^{\frac{1}{2}} + 2\log\left(\frac{K^2}{2\delta}\right)}  + 2\lambda^{1/2}S. 
	\end{align*}
	
	By using \cref{lem:detVt} to upper bound $det(\oVt)$, where $t>s$ with $a = 1$, and $n_{ij}^t \le t$, we get
	\begin{align}
		\label{equ:minEigenValLambda}
		\norm{\Te - \Ts}_{\oVt} \le \frac{\sigma\left(1-\frac{\lambda}{\lambda_{\min}(\Vt)}\right)^{-\frac{1}{2}}}{\kappa} \sqrt{{d^\prime}\log\left(1 + \frac{n_{ij}^t}{d^\prime\lambda} \right) + 2\log\left(\frac{K^2}{2\delta}\right)} + 2\lambda^{1/2}S.
	\end{align}
	
	As $t>s$ such that $\lambda_{\min}(V_{ij}(s)) \ge 2\lambda$, we have
	\[\norm{\Te - \Ts}_{\oVt} \le \frac{2\sigma}{\kappa} \sqrt{\frac{d^\prime}{2}\log\left(1 + \frac{n_{ij}^t}{d^\prime\lambda} \right) + \log\left(\frac{K^2}{2\delta}\right)} + 2\lambda^{1/2}S  = \beta_{ij}^t. \qedhere \]
\end{proof}

Note that if $\lambda_{\min}(V_{ij}(s)) < \lambda$ then $\left(1-\frac{\lambda}{\lambda_{\min}(\Vt)}\right)^{-\frac{1}{2}}$ is not well defined and the bound given in \cref{lem:theta_est_glm2} does not hold. Therefore, $\lambda_{\min}(V_{ij}(s))$ need to be at least greater than $\lambda$. Let $m^\prime \doteq C\lambda_\Sigma^{-2}\left(d+\log(k^2/2\delta)\right) + 4\lambda_\Sigma^{-1}\lambda$ where $C>0$ is the universal constant. Recall $R_{max} \doteq \max_{i\in [K], x\in \mathcal{X}}$ $ \left[C_i +  \gamma_{i}(x) - \left( C_{i^\star} + \gamma_{i^\star}(x)\right)\right]$, where $i^\star$ is the optimal arm for a context $x$. Now we state the regret bounds for \ref{alg:CUSS_GLM2}.
\begin{restatable}{thm}{cumRegGLM2}
	\label{thm:cum_reg_glm2}
	Let $\theta \in \TCWD$, $\lambda>0$, $\delta \in (0,1)$, Assumption 1 holds, and $\xi_h = \min\limits_{t \ge 1} \xi_{i^\star_t h}(x_t)$. Then with probability at least $1-2\delta$, the regret of \ref{alg:CUSS_GLM2} for $T > m^\prime$ contexts is upper bounded as
	\begin{align*}
		\Regret_T &\le R_{max}\Bigg(m^\prime + \sum_{h=2}^{K} \Bigg(\hspace{-1mm} \Bigg(\frac{C_1\sqrt{d^\prime} + C_2\sqrt{\log \left(\frac{K^2}{2\delta}\right)}}{\lambda_{\Sigma}}\Bigg)^2  \hspace{-2mm} + \frac{32\lambda}{\lambda_{\Sigma}}\left(\frac{k_\mu\sigma}{\xi_{h}\kappa}\right)^2 \\
		&\qquad 
		\Bigg(\sqrt{\frac{d^\prime}{2}\log\left(1 + \frac{T}{d^\prime\lambda} \right) + \log\left(\frac{K^2}{2\delta}\right)} + 2\lambda^{1/2}S\Bigg)^2\Bigg)\Bigg).
	\end{align*} 
\end{restatable}
\begin{proof}
	The proof follows similar steps as \cref{thm:cum_reg_glm} by replacing $m$ by $m^\prime$ and $\alpha_{ij}^T$ by $\beta_{ij}^T$. Using $\beta_{ij}^T = \sqrt{\frac{d^\prime}{2}\log\left(1 + \frac{T}{d^\prime\lambda} \right) + \log\left(\frac{K^2}{2\delta}\right)} + 2\lambda^{1/2}S$ completes the proof. 
\end{proof}

%% file: supp_experiments.tex

Since the parameter of each arm (classifier) is known to us (but not to the algorithm), the optimal arm $i^\star_t$ can be computed for every context. Therefore, we can also calculate the fraction of contexts for which $\WD$ property holds to a given cost vector. To verify $\WD$ property for a given context $x_t$, we first compute disagreement probability for each $(i,j)$ pair of classifiers as\footnote{For computing disagreement probability, we assume that the feedback of any arm is independent of the feedback of other arms. Note that \ref{alg:CUSS_WD} does not need such an assumption.}
\[
\pijt = \mu(x_t^\top\theta_i)(1-\mu(x_t^\top\theta_j)) + \mu(x_t^\top\theta_j)(1-\mu(x_t^\top\theta_i)).
\]

When all $\pijt$ values and $i^\star_t$ are known, we can check whether a context $x_t$ satisfies $\WD$ property or not by using \cref{equ:WDProp}. For all problem instances derived from the synthetic dataset, the cost vector and the fraction of contexts for which $\WD$ property holds are given in Table \ref{table:sysDatasetCases}. 

\begin{table}[H]
	\centering
	\setlength\tabcolsep{5pt}
	\setlength\extrarowheight{4pt}
	\begin{tabularx}{0.828\textwidth}{|p{2cm}|p{1cm}|p{1cm}|p{1cm}|p{1cm}|p{1cm}|p{2cm}|}
		\hline
		\textbf{PI/Classifiers} & Clf. 1        & Clf. 2        & Clf. 3  & Clf. 4 & Clf. 5     & $\WD$ fraction     \\ 
		\hline
		Costs for PI $1$ & 0.01 & 0.02  & 0.032  & 0.05 & 0.55 & 0.997\\ 
		\hline
		Costs for PI $2$ & 0.01 & 0.02  & 0.032  & 0.05 & 0.6 & 1.0\\ 
		\hline
		Costs for PI $3$ & 0.01 & 0.02  & 0.032  & 0.05 & 0.65 & 1.0\\
		\hline
		Costs for PI $4$ & 0.01 & 0.02  & 0.032  & 0.05 & 0.7 & 1.0\\
		\hline
	\end{tabularx}
	\vspace{2mm}
	\caption{Details of different problem instances (PIs) derived from synthetic datasets.}
	\label{table:sysDatasetCases}
\end{table}

\textbf{Heart Disease dataset:} Each sample of the Heart Disease dataset has $12$ features. We split the features into three subsets and train a logistic classifier on each subset. We associate 1st classifier with the first $7$ features as input that include cholesterol readings, blood-sugar, and rest-ECG. The 2nd classifier, in addition to the $7$ features, utilizes the thalach, exang, and oldpeak features; and the 3rd classifier uses all the features. For performance evaluation, the different values of costs are used in three problem instances for both real datasets, as given in Table \ref{table:real_dataset}. The PIMA diabetes dataset has $768$ samples, whereas the Heart Disease dataset has only $297$ samples. As $5000$ contexts are used in our experiments, we select a sample in a round-robin fashion and give it as input to the algorithm.
\begin{table}[H]
	\centering
	\setlength\tabcolsep{5pt}
	\setlength\extrarowheight{4pt}
	\begin{tabularx}{0.985\textwidth}{|p{2cm}|p{0.81cm}|p{0.81cm}|p{0.81cm}|p{1.8cm}|p{0.81cm}|p{0.81cm}|p{0.81cm}|p{1.8cm}|}
		\hline
		\multirow{2}{*}{\parbox{2cm}{\bf Values/ \newline Classifiers}} &\multicolumn{4}{|c|}{\bf PIMA Indian Diabetes Dataset}&\multicolumn{4}{|c|}{\bf Heart Disease Dataset} \\ 
		\cline{2-9} 
		&Clf. 1 &Clf. 2&Clf. 3&WD Fraction &Clf. 1 &Clf. 2&Clf. 3&WD Fraction\\ 
		\hline
		Costs for PI $1$& 0.01 & 0.25 & 0.5 & 0.0692 &0.01 & 0.25 & 0.5 &0.1384\\ 
		\hline
		Costs for PI $2$& 0.01 & 0.3 & 0.5 & 0.1192 &0.01 & 0.3 & 0.5 &0.1454\\ 
		\hline
		Costs for PI $3$& 0.01 & 0.35 & 0.5 & 0.2204&0.01 & 0.35 & 0.5 &0.2426\\ 
		\hline
	\end{tabularx}
	\vspace{2mm}
	\caption{Details of different problem instances (PIs) derived from real datasets.}
	\label{table:real_dataset}
\end{table}

\textbf{Experiments Results:} Through our experiments, we show that the stronger the $\CWD$ property (large value of $\xi$) for the problem instance, it is easier to identify the optimal arm and, hence, has lower regret, as shown in \cref{fig:xi_regret}. We also compare the performance of \ref{alg:CUSS_WD} with three baseline policies on problem instances derived from the Heart Disease dataset (same as the PIMA Indian Diabetes dataset). As expected, we observe that \ref{alg:CUSS_WD} outperforms the baseline policies, as shown in \cref{fig:heart}. Note that we used $\delta=0.05$ and $\sigma=0.1$ in all experiments.
\begin{figure}[H]
	\captionsetup[subfigure]{justification=centering}
	\begin{subfigure}[b]{0.495\linewidth}
		\includegraphics[width=\linewidth]{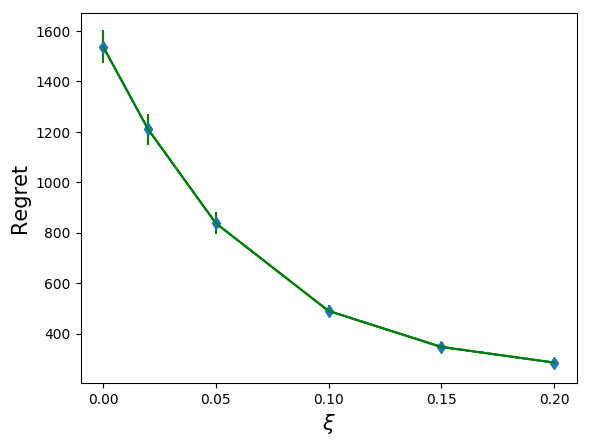}
		\caption{Regret v/s $\CWD$ property $(\xi)$.}
		\label{fig:xi_regret}
	\end{subfigure}
	\quad
	\begin{subfigure}[b]{0.495\linewidth}
		\includegraphics[width=\linewidth]{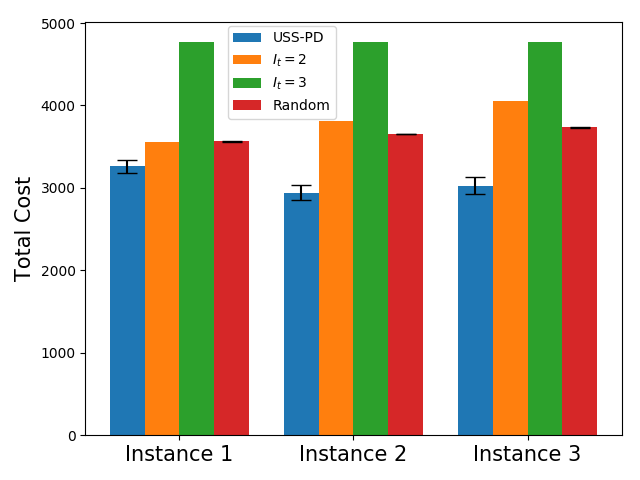}
		\caption{Total cost for PIMA Indian Diabetes dataset.}
		\label{fig:heart}
	\end{subfigure}
	\caption{\small Performance of \ref{alg:CUSS_WD}.}
	\label{fig:moreExp}
\end{figure}

\subsection{Realizable Setting}

We consider the realizable case where all contexts satisfy \cref{equ:disProb} (by fixing $\theta_{ij}$ for each $(i,j)$ pair of arms) and $\WD$ property. Since $\WD$ holds, we can use \cref{lem:Bx} for finding the optimal arm. Note that the mean loss cannot be computed for this setting as we set parameters of disagreement probabilities instead of setting parameters for individual arms.  We use an upper bound on the regret to evaluate the performance of \ref{alg:CUSS_WD} on the Synthetic dataset, as shown in \cref{fig:realizable}. We repeat experiments $500$ times to get a tighter confidence interval. 
\begin{figure}[H]
	\captionsetup[subfigure]{justification=centering}
	\begin{subfigure}[b]{0.495\linewidth}
		\includegraphics[width=\linewidth]{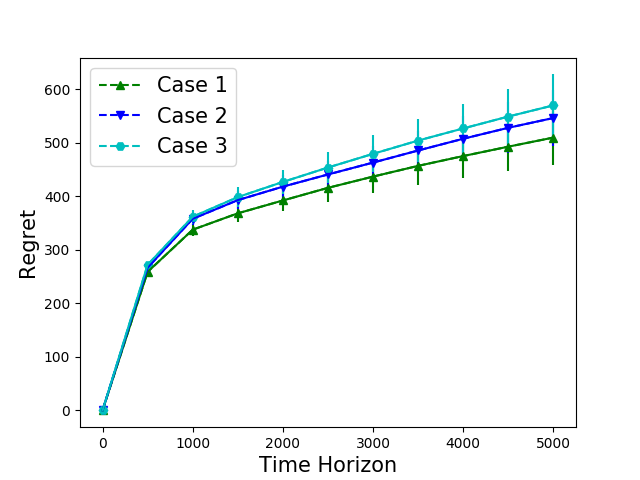}
		\caption{Synthetic dataset with $4$ classifiers where cost of using classifier $i$ in problem instance $j$ is $0.1 + (i-1)(0.09 + (j-1)0.01)$.}
		\label{fig:realizable4Cls}
	\end{subfigure}
	\quad
	\begin{subfigure}[b]{0.495\linewidth}
		\includegraphics[width=\linewidth]{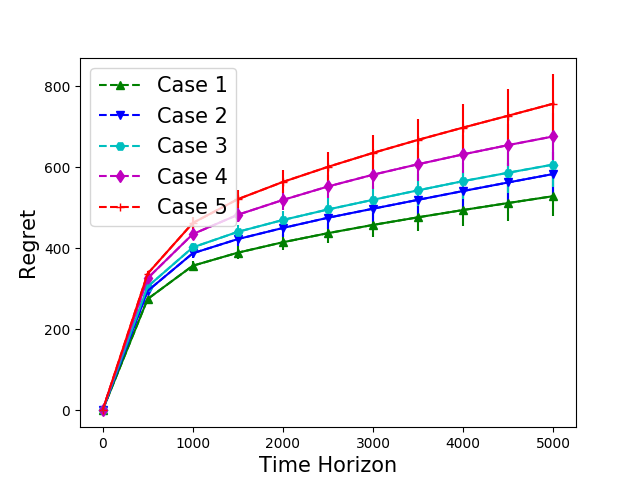}
		\caption{Synthetic dataset with $5$ classifiers where cost of using classifier $i$ for problem instance $j$ is $0.1 + (i-1)(0.06 + (j-1)0.01)$.}
		\label{fig:realizable5Cls}
	\end{subfigure}
	\caption{Performance of \ref{alg:CUSS_WD} for realizable setting where regret on y-axis is $\sum_{t=1}^T |C_{I_t} - C_{i^\star_t}| + \pist$, and it is an upper bound on the regret $\Regret_T$ defined in \cref{equ:cum_regret}. The value of $\xi$ largest for Case $1$, and it decreases for subsequent cases.}
	\label{fig:realizable}
\end{figure}

\paragraph{Regret used for Empirical Evaluation in Realizable Setting}~\\
Since the error-rate of arms is unknown, the regret defined in \cref{equ:cum_regret} can not be computed. Hence we define an alternative regret, which we call pseudo regret, as follows:
\begin{equation*}
	\Regret_T^s = \sum_{t=1}^T \left[C_{I_t} - C_{i^\star_t} + \pist \right].
\end{equation*}

It is easy to verify that the actual regret $\Regret_T$ is upper bounded by above regret $\Regret_T^s$ as shown follows:
\begin{align*}
	\Regret_T &= \sum_{t=1}^T \left[C_{I_t} + \gamma_{I_t}(x_t) - \left(C_{i^\star_t} + \gamma_{i^\star_t}(x_t)\right) \right] \\
	&= \sum_{t=1}^T \left[C_{I_t} - C_{i^\star_t} + \left(\gamma_{I_t}(x_t) -\gamma_{i^\star_t}(x_t)\right) \right]\\
	&\le \sum_{t=1}^T \left[C_{I_t} - C_{i^\star_t} + \pist \right] \hspace{5mm} \text{(Using \cref{lem:err_prob_contx})}\\
	\implies \Regret_T &\le \Regret_T^s.
\end{align*}

\subsection{Contextual Strong  Dominance }
We next introduce contextual strong dominance property of the problem instance. 
\begin{defi}[Contextual Strong  Dominance $(\CSD)$ property] 
	\label{def:CSD} 
	A problem instance is said to satisfy $\CSD$ property if for all contexts following is true:
	\eqs{
		Y^i = Y \mbox{ for some } i \in[K] \implies  Y^j = Y, ~~\forall j \in [K]\setminus [i].
	}
	We represent the set of all instances satisfies $\CSD$ property by $\TCSD$.
\end{defi}
The $\CSD$ property implies that if the feedback of an arm is the same as the true reward of a given context then, the feedback of all the arms in the subsequent stages of the cascade is also the same as the true reward of a given context.

When any problem instance satisfies $\CSD$ property, the value of $\Prob{\Yti = \Yt, \Yti \ne \Ytj|X=x_t} = 0$ for $j>i$. Therefore, for any $(i,j)$ pair of arms and context $\xt$ the following is true:
\eqs{
	\forall j>i, \gamma_{i}(\xt) - \gamma_{j}(\xt) = \Prob{\Yti \ne \Ytj| X= \xt}.
} 
The above equation implies that $\CWD$ property holds trivially for the problem instances that satisfy $\CSD$ property as the difference of mean losses is the same as the probability of disagreement between two arms(fix arm $i = \ist$ for given context $\xt$).

\subsection{Effect of adding more arms on $\WD$ property}  
The performance of \ref{alg:CUSS_WD} can deteriorate as we increases as the number of arms. This is because the fraction of contexts that satisfy $\WD$ property can decrease with the increase in the number of arms. To see that, consider a contextual USS problem instance with three arms where arm $1$ has cost $0.1$, arm $2$ has cost $0.2$, and arm $3$ has cost $0.3$. Let there be two contexts $x_1$ and $x_2$ such that classifier $2$ is an optimal classifier for context $x_1$ and classifier $3$ for the context $x_2$, and both contexts satisfy $\WD$ property. When a new arm is added at the end of the classifiers cascade without changing the optimal arm for the contexts, let $p_{24}^{(1)}$ be the disagreement probability for classifier $2$ and $4$ for context $x_1$ and $p_{34}^{(2)}$ be the disagreement probability for classifier $3$ and $4$ for context $x_2$. It is easy to verify that if cost of using classifier $4$ is less than $\min\{0.2+p_{24}^{(1)}, 0.3+p_{34}^{(2)}\}$ then both contexts will not satisfy $\WD$ property. 